\documentclass{article}

\usepackage{microtype}
\usepackage{graphicx}
\usepackage{subfigure}
\usepackage{booktabs} % for professional tables

\usepackage{hyperref}

\usepackage[accepted]{icml2024}

\usepackage{amsmath}
\usepackage{amssymb}
\usepackage{mathtools}
\usepackage{amsthm}

\usepackage[capitalize,noabbrev]{cleveref}

\theoremstyle{plain}
\newtheorem{theorem}{Theorem}[section]
\newtheorem{proposition}[theorem]{Proposition}
\newtheorem{lemma}[theorem]{Lemma}
\newtheorem{corollary}[theorem]{Corollary}
\theoremstyle{definition}

\newtheorem{assumption}[theorem]{Assumption}
\theoremstyle{remark}

\usepackage[disable,textsize=tiny]{todonotes}

\usepackage{enumitem}
\usepackage{microtype}
\usepackage{graphicx}
\usepackage{subfigure}
\usepackage{booktabs} % for professional tables
\usepackage{multicol}
\usepackage{comment}
\usepackage{soul}
\usepackage{amsmath}
\usepackage{amssymb}
\usepackage{mathtools}
\usepackage{amsthm}
\usepackage{bbm}
\usepackage{bm}
\usepackage{xcolor}
\usepackage{mathtools}
\usepackage{mathrsfs}  
\usepackage[export]{adjustbox}
\usepackage{multirow}
\usepackage{wrapfig}
\usepackage{enumitem}
\usepackage{lipsum}
\usepackage{array}
\usepackage{color, colortbl}
\usepackage[normalem]{ulem}

\definecolor{LightCyan}{rgb}{0.6,1,1}

\usepackage{pifont}% http://ctan.org/pkg/pifont

\DeclareMathOperator{\poly}{poly}

\newcommand{\mtx}{\bm} % bold matrix
\newcommand{\vct}{\bm} % bold vector

\newcommand{\E}{\mathbbm{E}}

\newcommand{\cat}[2]{
\begin{bmatrix}#1 \\ #2
\end{bmatrix}}

\newcommand{\blue}[1]{{{\textcolor{black}{#1}}}}

\newcommand{\ours}{MixPro}
\newcommand{\embmx}{\mathcal{E}_{\text{mixed}}}
\newcommand{\mixed}{\text{mixed}}
\newcommand{\prot}{\textsc{Pro}\textsuperscript{2}}

\newcommand{\unif}{\overset{\text{unif.}}{\sim}}

\newcommand{\spu}{\text{spu}}

\icmltitlerunning{Few-shot Adaptation to Distribution Shifts By Mixing Source and Target Embeddings}

\begin{document}

\twocolumn[
\icmltitle{Few-shot Adaptation to Distribution Shifts By Mixing\\ Source and Target Embeddings}



\icmlsetsymbol{equal}{*}

\begin{icmlauthorlist}
\icmlauthor{Yihao Xue}{yyy}
\icmlauthor{Ali Payani}{comp}
\icmlauthor{Yu Yang}{yyy}
\icmlauthor{Baharan Mirzasoleiman}{yyy}
\end{icmlauthorlist}

\icmlaffiliation{yyy}{Department of Computer Science, University of California, Los Angeles}
\icmlaffiliation{comp}{Cisco Systems Inc.}

\icmlcorrespondingauthor{Yihao Xue}{yihaoxue@g.ucla.edu}

\icmlkeywords{Machine Learning, ICML}

\vskip 0.3in
]




\begin{abstract}
% Machine learning models often perform poorly under distribution shift, when there is a discrepancy between the data distributions of their training and deployment environments.
% This necessitates models to be adapted to the target distribution before being deployed. 
Pretrained machine learning models  need to be adapted to distribution shifts when deployed in new target environments. 
When obtaining labeled data from the target distribution is expensive, few-shot adaptation with only a few examples from the target distribution becomes essential. In this work, we propose \ours, a lightweight and highly data-efficient approach for few-shot adaptation. \ours\ first generates a relatively large dataset by mixing (linearly combining) pre-trained embeddings of large source data with those of the few target examples. This process preserves important features of both source and target distributions, while mitigating the specific noise in the small target data. Then, it trains a linear classifier on the mixed embeddings to effectively adapts the model to the target distribution without overfitting the small target data. 
Theoretically, we demonstrate the advantages of \ours\ over previous methods. Our experiments, conducted across various model architectures on 8 datasets featuring different types of distribution shifts, reveal that \ours\ can outperform baselines by up to 7\%, with only 2-4 target examples.\looseness=-1

%Machine learning models pre-trained on a source data need to be adapted to distribution shift, when they are deployed in a new target environment.

%Theoretically, we show that \ours\ outperforms prior methods for few-shot adaptation, or training on only source or target distribution. Our experiments with various model architectures on 8 datasets, containing various types of distribution shift, show that \ours\ can outperform baselines by as much as 7\%, with only 2-4 target examples.%\looseness=-1

% \blue{revisions in respond to the main criticisms. (Dang, Yu) motivation of last layer retrainig: blue in par 2 of the intro, and par 2 of section 4. (Dang) comparison to mixup: third to last par Sec 4. (Dang) fixed s: 2nd to last par Sec 4. (Yu) DFR needs large data: 2nd par intro/2nd par sec 4, where we refer to both pro2 paper and our experiments for support. (Yu, Dang) results not promising: added new figure in experiments sec; added average plot; now have one section for using large validation for hp selection and another section for realistic cross validation. Emphasizing that our method consistently performs the best across different scenarios, while other methods may perform comparably on certain datasets but fail on others. } 
\end{abstract}

\section{Introduction}

Modern machine learning models often struggle to generalize well, when deployed in domains where the data distribution significantly differs from their source training data distribution \cite{quinonero2008dataset}. %This is because they learn features in a way that allows them to generalize best on their training data. 
Thus, before deployment in a new domain, they need to be adapted to the target distribution. When abundant data from the target domain is available, one can simply fine-tune the model on the target data to improve its performance. %on the target distribution.
% This preserves the relevant features and learns new ones that are beneficial for generalization on the target domain. 
Nevertheless, in many real-world scenarios only a limited number of examples from the target domain is available. 
For example, data for training autonomous vehicles %to operate 
in severe weather conditions, %differing from those in existing datasets, 
are not only rare in certain geographical areas but also pose safety risks for data collection. In medical diagnosis, collecting data for rare diseases %with few patient records 
is often challenging. In such scenarios, fine-tuning on the small target data fails, by overfitting the few available examples instead of learning their features in a generalizable manner. \textit{Few-shot adaptation} of a model to a new domain %, when only a few target examples are available, 
requires developing data-efficient methods that can effectively adapt the model with only \textit{a few} examples from the target domain. 

Recent studies revealed that neural networks learn versatile features from the training data in their {penultimate layer}
\cite{kirichenko2022last, izmailov2022feature, lee2022surgical, mehta2022you, rosenfeld2022domain},
%This observation led to the recent approach of 
and retraining only the {last layer} using the target data %adjust the weights of
can effectively re-weight features to improve generalization on the target distribution. This approach is extremely lightweight and %achieves excellent
%state-of-the-art
%performance, even 
outperforms end-to-end fine-tuning of the whole model on the target data (Appendix C.1, \citealt{kirichenko2022last}), and therefore has recently become a popular technique %in many recent studies that
to deal with different kinds of distribution shift \cite{kirichenko2022last,mehta2022you,rosenfeld2022domain,izmailov2022feature,chen2023project,qiu2023simple}.
%This observation has been recently leveraged to develop techniques that re-weight only the last layer of a network to adjust the weights of features, to improve generalization under distribution shifts due to spurious correlations or domain shift \cite{kirichenko2022last,mehta2022you, rosenfeld2022domain}.
Nevertheless, {last-layer} retraining %this approach 
% requires a relatively large data from the target distribution and %is not data-efficient. 
is not the most data-efficient technique and can perform poorly for few-shot adaptation
% performs poorly otherwise, 
\cite{chen2023project}. %and our experiments \blue{(Figures XX)}.
To address this, \citet{chen2023project} proposed \prot\ that finds a %compact and 
diverse set of %useful 
features from the %abundant 
source data and re-weights them %using
by training a linear model on
the few available target data. But, as finding the set of compact features is done independent of the target data, this approach may miss capturing %important features that are 
important relevant features for the target domain, and yields sub-optimal performance. %\looseness=-1
% , and is not effective when there is more discrepancy between source and target distributions. 

In this work, we develop a highly data-efficient method, \ours, that takes advantage of the few available target examples %from the  domain 
in addition to the abundant source data, to re-weight the last layer of a pre-trained model, effectively adapting it to the target distribution. \ours\ first constructs a large new dataset by \textit{mixing} embedding of every source examples with that of a randomly chosen target example in the same class, via taking their weighted linear combination. 
% In doing so, the new constructed dataset contains %a considerable amount of 
% information about the target domain, \textit{without} losing the valuable information in the source data. 
Then, it %re-weights the last layer of the network by 
trains a linear model on the new dataset of mixed embeddings. 
In the new dataset, every example is a combination of a target example with a distinct source example. Therefore, the model trained on the mixed embeddings learns the target information, without overfitting (particular noise in) the few target examples. Besides, it can take advantage of the relevant information of the source data.
In doing so, \ours\ effectively adapts the model to the target domain. 
% At the same time, training on the large mixed data prevents overfitting the noise in the few available target examples. 
% Hence, the relevant information in the source data that do not present in the few target examples can be easily leveraged by the linear probe to achieve superior performance.

We theoretically validate the effectiveness of our approach in two scenarios. %In the first example
Firstly, we demonstrate that when important features for the target distribution appear irrelevant in the source distribution, \prot\ \cite{chen2023project} fails to learn these features, resulting in poor generalization on the target. In contrast, \ours\ effectively learns the important features in a data-efficient manner. Secondly, we adopt a model for subpopulation shift \cite{sagawa2020investigation} and examine \ours's performance in high-dimensional asymptotics. We show that \ours\ enables effective 
% combines the best \blue{aspects} of both the source and target data, thereby 
learning from the target distribution while using the source data to prevent overfitting the small target. 
% We illustrate the trade-off involved in selecting the mixing weight, which naturally leads to showing how 
In doing so, \ours\ outperforms %both DFR 
last-layer training on only target
\cite{kirichenko2022last} %and training solely on 
or source data.  
Moreover, our analysis demonstrates how the severity of the shift, the noise level in the target data, and the target sample size influence the optimal mixing weight. 

Empirically, we conduct extensive experiments on 8 datasets, including 3 subpopulation shift datasets -- Waterbirds \cite{sagawa2019distributionally}, UrbanCars \cite{li2023whac}, bFFHQ \cite{kim2021biaswap} -- and 5 domain generalization datasets -- Camelyon17\cite{koh2021wilds}, PACS \cite{li2017deeper}, VLCS \cite{fang2013unbiased}, Office-Home \cite{venkateswara2017deep} and Terra Incognita \cite{beery2018recognition}. We show that \ours\ outperforms existing baselines, achieving superior performance even with very few (2 to 16 per class) target data, across various datasets and model architectures. It can achieve a maximum improvement of 7\%, and on average, outperforms baselines by 4.3\%/3.9\% for 2-shot/4-shot adaptation across datasets. Finally, we show that \ours\ remains the best method when hyperparameters are selected using cross validation with only a few target data. \looseness=-1

% A common explanation for the failure of deep networks to generalize out-of-distribution is that they fail to recover the “correct” features. But this notion is challenged recently. ... cite works about last layer training. The ERM base model has already learned rich features; the current bottleneck is not feature learning, but rather how to effectively adapt the last layer. Thus, learning predictors on existing features under distribution shift is a promising direction for future research.

% ... In practice, it is more common to have \textbf{a limited number of labeled target data points} rather than \textbf{having none} or \textbf{having a large number of them}. Therefore, an important direction is how to adapt the last layer in a sample-efficient manner.

% \textbf{we find that mixing the source embeddings and target embeddings is essential in achieving superior robustness, even if the base model is already trained on the source data}
\section{Related Work}

\textbf{Distribution shift.} Distribution shift, or domain shift, refers to a scenario where a model is trained on data from one distribution but is expected to generalize to test data from different distributions. 
%This concept generally branches into two categories: subpopulation shift and domain generalization. As defined in \cite{koh2021wilds}, subpopulation shift occurs when training data consists of data collected from various domains or environments, while the test data contain the same domains, but their proportions differ from the training data. Domain generalization refers to cases where the test data come from entirely different domains, unseen in the training data. There are numerous works in both areas \cite{},
Prior work primarily focuses on two settings: zero-shot generalization, which involves training a classifier on source data without seeing target data and expecting it to perform well on the target distribution \cite{tzeng2014deep,ganin2016domain,zhai2019large,yosinski2014transferable,sagawa2019distributionally,arjovsky2019invariant,creager2021environment,kornblith2019better,zhang2022contrastive,wortsman2022robust,sharif2014cnn,nam2020learning,oquab2014learning,liu2021just,kumar2022fine,pagliardini2022agree,lee2022diversify}; and test time adaptation, where the trained model is additionally updated upon seeing unlabeled test data from the target distribution \cite{sun2020test,varsavsky2020test,iwasawa2021test,wang2020tent,zhang2021adaptive,gandelsman2022test}. These differ from our setting since we consider a few-shot case where labeled target data are available but very few. Additionally, these works %are unaware
do not take advantage
of the recent discovery that naively training a neural network can already learn generalizable features in the representation layer, thus do not address the newly identified bottleneck regarding last layer retraining, which we discuss in the next paragraph.

\textbf{Last layer retraining for distribution shifts.} A common intuition, shared in aforementioned works, is that the failure of deep networks to generalize to out-of-distribution data stems from their inability to learn generalizable features
%correct features. 
from their training data. 
However, this notion has been recently challenged by works including \cite{kirichenko2022last, izmailov2022feature, lee2022surgical, mehta2022you, rosenfeld2022domain}. These studies demonstrate that the model trained on source data has already learned rich features, and that simply retraining the last linear layer 
%(linear probing)
(also called DFR \cite{kirichenko2022last})
with target data can already achieve
excellent
%state-of-the-art 
performance. This suggests that the real bottleneck is not in feature learning, but rather in how to effectively re-weight the last layer features. However, in these works, the last layer is retrained with a large data from the target distribution, making them impractical in many real-world scenarios where target data are typically scarce. Our work specifically targets this direction, seeking more {data-efficient} solutions than standard linear probing.\looseness=-1

\textbf{Few-shot adaptation and sample efficiency.} 
%Realizing that in practice, it is more common to have \textit{a limited number of labeled target data points} rather than \textit{none} or \textit{a large number}, training the last layer in a sample-efficient manner becomes an emerging challenge. 
In scenarios where labeled target data are available but scarce, adapting the last layer to the target distribution in a data-efficient manner presents a significant challenge. The study most relevant to our paper is \cite{chen2023project}, which considers the same setting as ours, where the number of available target data for linear probing is very small (e.g., 2 to 32 per class). They propose \prot\ that first learns a projection that maps the source data embeddings to orthogonal informative directions. Then, it passes the target data embeddings through the learned projection and perform linear probing using these projected embeddings.
%map the original embeddings to orthogonal informative directions using the large volume of source data at hand, and then perform linear probing on the projected embeddings of the few target examples.
%then perform linear probing using the few target data points on the projected embeddings. 
% While this method performs well under mild distribution shifts, it is less effective under more severe distribution shifts 
\cite{teney2022evading}'s method shares a similar intuition but does so with an additional loss term instead of explicitly enforcing orthogonality. However, we find that in both methods, the projection layer learned on the source data may miss %\blue{or underrepresent} 
directions in the embeddings that are important for the target distribution, as we will theoretically demonstrate.
%However, we identify a caveat in this method, supported by theoretical analysis: under severe distribution shifts, the projected embeddings may miss or underweight important directions, rendering their method ineffective. 
In contrast, our method does not suffer from this issue. We include 
%\cite{chen2023project,teney2022evading} 
these two methods as baselines in our experiments.
\cite{zhu2022progressive} and \cite{zhang2022few} also utilize both source and target data to achieve few-shot adaptation, albeit in different ways. However, their methods modifies the pretrained encoder, which could lead to overfitting on the limited dataset and degrades representation quality. We provide a comparison and further discussion of these methods in Appendix \ref{apdx: additional_comp}.

% adapts pretrained models using both image-level and representation-level mixup, with the mixing ratio progressively adjusted, and optimization performed using the MAML meta-learning framework. However, given the smaller size of the source data compared to the pre-training data, updating encoders with source data could degrade representation quality by overfitting to the limited dataset. As a result, their method is not as data-efficient as ours, as 
\looseness=-1

\section{Problem Formulation}

We consider adapting a model to make accurate predictions on a shifted distribution in the few-shot setting, where we have access to a large source data but only a few examples from the shifted target distribution.
% In our problem setting, the task is to adapt a model to ensure it provides accurate predictions when faced with distribution shift, while having access to only a minimal amount of examples from the target distribution.

%\yx{todo: mention that labels are the same}
Formally, we have a source distribution $\mathscr{P}^{s}$ and a target distribution $\mathscr{P}^{t}$. The two distributions differ, for instance, due to data being collected at different times, from various regions, or in distinct environments.
Similar to \cite{chen2023project}, we assume that the supports of the labels are the same in both distributions, while the supports of the inputs may or may not be the same. This encompasses both scenarios of subpopulation shift and domain generalization, as defined in \cite{koh2021wilds}.
The source dataset $\mathcal{D}^s=\{\vct{x}_i^s, \vct{y}_i^s\}_{i=1}^n$, is composed of $n$ examples from a set of classes $\mathcal{C}$ drawn from the source distribution $\mathscr{P}^s$, while the target dataset $\mathcal{D}^t=\{\vct{x}_i^t, \vct{y}_i^t\}_{i=1}^m$ consists of $m$ examples from $\mathcal{C}$ drawn from the target distribution $\mathscr{P}^t$. In the few-shot setting we consider, $m \ll n$, where %and the value of 
$m$ is very small (e.g., in our experiments, we consider 2 to 32 examples per class). 
After adaptation, the model is evaluated on a held-out test set %target dataset 
from the target distribution.

% \yx{add par about differnece with zeroshot etc}

% Following \cite{chen2023project}, we assume access to a pretrained backbone model $f: \mathcal{X}\rightarrow \mathbbm{R}^d$ that extracts meaningful features from the inputs. The goal is to learn a predictor $g: \mathbbm{R}^d \rightarrow \mathcal{Y}$ on top of the embeddings \blue{output by the pretrained model} and ensure that it performs well on the \textit{target} distribution $\mathscr{P}^t$. 
% In a classification task, the goal is to $\max_{g} \E_{(\vct{x}, y) \sim \mathscr{P}^t} [ \mathbbm{1}_{y = g(f(\vct{x}))} ]$, where $\mathbbm{1}_{(\cdot)}$ denotes the indicator function. To be consistent with prior work \cite{kirichenko2022last,chen2023project}, in our experiments we consider $g$ to be a linear classifier. However, we note that our proposed method is applicable to any kind of (non-linear) predictor $g$.

Note that our setting differs from prior works on zero-shot generalization \cite{sagawa2019distributionally,kumar2022fine,wortsman2022robust} under distribution shift, where the model is exclusively trained on the source data $\mathcal{D}^s$ and directly evaluated on the target distribution. 
% We show that the few-shot setting is more realistic since
A very small amount of target data that can be realistically obtained in many settings, %it can be leveraged to best adapt the model to the target domain,
is often essential to deal with arbitrary distribution shift,
as we will confirm in our work.
%and necessary for addressing challenging distribution shifts. 
Closer to our setting are \cite{kirichenko2022last,izmailov2022feature,lee2022surgical}, which perform linear probing with target data on a pre-trained model that is fine-tuned on the source data. 
% but they used very large target datasets, which may not be feasible in real-world scenarios. 
% We believe that in practice, target data are usually available but scarce. 
However, such methods are not very data-efficient, and 
may perform poorly when target data is very small. We specifically consider the adaptation with \textit{a few} target examples, which is also recently considered in \cite{chen2023project}. \looseness=-1%, the exact same setting as ours is considered. %; therefore, we include their method as a baseline in our experiments.

\section{\ours\ (Mix \& Probe): Data-efficient Few-shot Adaptation to Target Domains}
\label{sec: method}

In this section, we first discuss the challenges and considerations of few-shot adaptation to distribution shift, and then introduce our method, \ours, to overcome these challenges.\looseness=-1

\textbf{Challenges \& considerations.} Adapting a model to new domains with only a few examples from the target domain is very challenging for the following two reasons: 
First, while last-layer retraining on large target data outperforms end-to-end fine-tuning %and %achieves state-of-the-art 
% has been widely adopted in various scenarios 
\cite{kirichenko2022last,izmailov2022feature,rosenfeld2022domain,yang2023change},
% First, the model cannot be fine-tuned in an end-to-end manner, as fine-tuning an overparameterized network on the few target examples results in major overfitting and yields a poor performance.  While re-training the last layer \cite{kirichenko2022last} is more data-efficient compared to end-to-end fine-tuning, 
%it still 
this approach is not very data-efficient and poses a risk of
%it requires a sufficiently large target dataset and}
%suffers from 
overfitting when only \textit{a few} examples from the target distribution are available (\textit{c.f.} \cite{chen2023project}, and our figures in Sec \ref{sec: exp} and \ref{apdx: addtional_exp}). Second, adaptation to the target domain should %benefit from \textit{both source and target} data. 
take maximum advantage of the few available target examples to achieve optimal performance.
%That is, only leveraging the source data for few-shot adaptation as done
Indeed, the prior work that first find a diverse set of features from the source and re-weight them using the target data \cite{teney2022evading,chen2023project} may yield sub-optimal performance. This is because 
some features that are important for the target domain, but seem unimportant in the source data, may be missed.
%some important information contained in the source data that are relevant for the target domain may be missed. 
%when finding a small subset of diverse features. 
We will confirm this theoretically in Sec. \ref{sec: theory}. \looseness=-1

\subsection{Linear probing on mixed source and target embeddings}

We present our method, \ours\, that 
%successfully 
effectively
adapts a model with only a few data from the target distribution.\looseness=-1

\textbf{Key idea.} The key idea of our method is to take advantage of the large source data and a few target examples to \textit{construct} a new \textit{large} dataset that contains generalizable information about the target domain. If this can be done, then last-layer retraining (linear probe) on the new constructed data achieves a superior performance as: (1) it does not overfit the (noise in the) few target examples, and (2) as the new constructed data contains information of both the source and target domains, the linear probe can take advantage of all the relevant information during training to achieve superior performance on the target distribution. 

To construct a large dataset that contains information about the target data, our main idea is to leverage the large available source data and incorporate information about the target into it. To do so, we use a pre-trained backbone model $f:\mathcal{X}\rightarrow \mathbb{R}^d$ to map the source and target examples to a $d$-dimensional embedding space. Then, we create a new large embedding dataset by taking every source example and \textit{mix} its embedding with the embedding of a randomly-chosen target example in the same class. For mixing the embeddings, we simply take linear combinations of the embeddings of the source and target examples.  
In the new dataset, every example is a combination of a target example with a distinct source example. Therefore, the model trained on the mixed embeddings learns the target information in a more generalizable manner, without overfitting (particular noise in) the few target examples.
% it does not inherit the particular noise contained in the few target examples. 
%As every source example is fused with a target example, the new dataset contains a considerable amount of information about the target domain. At the same time, as the target embeddings are mixed with the source embeddings, the model cannot overfit the noise in the few target examples.
In addition, the new dataset still contains the information from the source data. Hence, the valuable information of the source data that are relevant to the target domain but do not present in the few target examples can be easily leveraged by the linear probe to achieve superior performance.

\textbf{Formal description. } 
Formally, for source dataset $\mathcal{D}^s=\{\vct{x}_i^s, \vct{y}_i^s\}_{i=1}^n$ with $n$ examples and a target dataset $\mathcal{D}^t=\{\vct{x}_i^t, \vct{y}_i^t\}_{i=1}^m$, with $m\ll n$ examples, \ours\ has the following two steps:

\textbf{(1) Mixing source \& target.} we construct a new dataset of embeddings with their labels, expressed as follows:
\begin{align}
   \embmx = \{  (1-s) f(\vct{x}_i^s) + s f(\vct{x}_{j_i}^t), ~ ~ y_i^s  \}_{i=1}^n,
\end{align}
where $j_i \overset{\text{unif.}}{\sim} \{j~|~ y_j^t=y_i^s, j\in[m] \}$, meaning that each $j_i$ is uniformly randomly sampled from the indices of the target data whose labels equal $y_i^s$. That is, for each source example, we randomly select one target example with the same label and take a weighted average of their embeddings to create a new embedding, where $s$ is the weight for the target example and serves as a hyperparameter. 

\textbf{(2) Linear probe on mixed embeddings.} Finally, we train a linear classifier $g$ on the mixed embeddings by minimizing loss function $l(\cdot, \cdot)$: \vspace{-1mm}
\begin{align}
    \label{eq: train_on_mixed}
    \min_{g} \hat{\E}_{ (\vct{z}, y)\in\embmx } l(g(\vct{z}), y ),
\end{align}
where $\hat{\E}$ denotes the empirical expectation. 

% Note that in contrast to Mixup \cite{zhang2017mixup}, our method does not mix labels of the source and target examples. Effectively, the goal of the two methods are entirely different. 
% one should not confuse our method with Mixup \cite{zhang2017mixup}. 
\textbf{\ours\ vs (manifold) Mixup.} We highlight the difference between \ours\ and (Manifold) Mixup \cite{zhang2017mixup,verma2019manifold}.
Mixup improves the in-distribution generalization by training the model on mixed \textit{inputs and labels} of pairs of examples, using \textit{randomly sampled weights}. In contrast, our method is specifically designed for data-efficient adaptation and involves mixing embeddings of source and target data with a \textit{fixed} weight, \textit{without} mixing labels of examples in difference classes. We will show \ours\ outperforms Mixup in our experiments in Section \ref{sec: exp}. \looseness=-1

\textbf{Using a fixed $s$.} %One might wonder if we could adopt an approach similar to Mixup, where the mixing weight is randomly sampled each time we mix, for instance, from a Beta distribution. However, 
We note that 
randomly sampling $s$, e.g. from a Beta distribution as is done in Mixup, 
is not effective, as %setting the value of $s$ appropriately is essential for optimal performance. The value of 
$s$ determines the proportions of source and target data in the mixture. Hence, it should be set appropriately for every dataset. When $s$ is randomly selected from a Beta distribution, its expected value is $0.5$,  implying equal weights for source and target data, which can lead to suboptimal performance (\textit{c.f.} Figure \ref{fig: subpopshift_curve}).\looseness=-1

\textbf{The selection of $s$.} %we first note that all existing techniques,
Prior work \cite{kirichenko2022last,chen2023project,teney2022evading}, involve hyperparameter tuning %and typically select them 
based on a hypothetical large validation set from the target distribution. %, which nearly equates to oracle selection. 
Indeed, addressing hyperparameter selection in %a realistic way for few-shot adaptation 
the few-shot scenario has not been explored before. In Sec. \ref{sec: cross_val}, we demonstrate that $2$-fold cross-validation with the few available target examples, %a fully realistic approach, 
can lead to reasonable hyperparameter selection. Our method outperforms prior work, whether using this strategy or a hypothetical large validation set. 
%Under this selection strategy, our method outperforms prior work.
\looseness=-1

\section{Theoretical Analysis}\label{sec: theory}

In this section, we provide %two theoretical examples
theoretical analysis in two scenarios that comprehensively showcase the advantages of \ours. Firstly, we compare \ours\ with \prot\ \cite{chen2023project}, which learns a projection of pre-trained embeddings of source data before performing linear probing on target data. We demonstrate that \prot\ can overlook important features for the target distribution, whereas our method, which directly performs linear probing on a mixture of source and target data, does not have this issue. Secondly, we consider a subpopulation shift model and demonstrate the trade-off in selecting the mixing weight, as well as the advantage of our method over using solely source or solely target data.\looseness=-1

\subsection{A case study of domain generalization: the advantage of \ours\ over %$\pro$
prior work
}\label{sec: theory_domain_gen}

% \yx{what is the idea of the theory;}
Recent works \cite{teney2022evading, chen2023project} suggested leveraging the large source data to learn a projection from the original embedding space to a lower-dimensional space that preserves a small and diverse (orthogonal) set of features from the source data. %, effectively reducing noise. 
%Training a linear model
Linear probing on the target domain using these orthogonal directions has been found to be more data-efficient. But, the drawback of this approach is that the projection, being learned solely from source data, may disregard directions that are less significant in the source but critical for the target distribution. Next, we show that such methods may result in sub-optimal performance.\looseness=-1

In our analysis, we make assumptions about the distribution of embeddings on which different methods are applied.
In the source distribution, each label $y$ is uniformly drawn from $\{-1, 1\}$ and the corresponding embedding $\vct{z}$ is given by $\vct{z} = y\vct{v} + \vct{\xi}$,
% \vspace{-1mm}
% \begin{align}
% \nonumber
%     \vct{z} = y\vct{v} + \vct{\xi},
% \vspace{-.5cm}
% \end{align}
where \( \vct{v}\overset{\text{uni}}{\sim} \{ \vct{v}_1, \vct{v}_2 \} \), \( \vct{\xi}=\xi\vct{v}_3 \) with \( \xi\sim\mathcal{N}(0, \sigma) \). $\vct{v}_1, \vct{v}_2, \vct{v}_3$ are orthogonal unit vectors. Simply put, in the embedding space, the label information is carried in two directions, \( \vct{v}_1 \) and \( \vct{v}_2 \), with each embedding encoding label information in one of these two directions, while having some noise in the third direction, \( \vct{v}_3 \).

To model the shift in distribution, we consider the following target data distribution where each embedding is given by $\vct{z} = y\vct{v} + \vct{\xi}$
%\vspace{-1mm}
% \begin{align}
% \nonumber
%     \vct{z} = y\vct{v} + \vct{\xi},
%     \vspace{-4mm}
% \end{align}
where \( \vct{v}\unif \{ \vct{v}_2, \vct{v}_3 \} \), \( \vct{\xi}=\xi\vct{v}_1 \) with \( \xi\sim\mathcal{N}(0, \sigma) \). The shift is such that, in the target distribution, the label information is carried in \( \vct{v}_2 \), which is the shared part with the source distribution, but it differs in that it can also be carried in \( \vct{v}_3 \) and never in \( \vct{v}_1 \). Instead, \( \vct{v}_1 \) consists only of noise in the target distribution.

Assuming we have access to \( n \) source data embeddings in dataset $\mathcal{E}^s$  and \( m \) target data points in dataset $\mathcal{E}^t$ . To model the case where the source data is very large, we let \( n\rightarrow \infty \).

This simple model allows seeing the failure mode of previous methods, including  \cite{chen2023project,teney2022evading}. Specifically, we consider \prot\  \cite{chen2023project}, while noting that \cite{teney2022evading} shares a similar underlying mechanism, and is nearly equivalent to \prot\ for linear models. We defer detailed discussion to \blue{Appendix \ref{apdx: teney}}. \looseness=-1

\prot\ first learns a projection layer $\mtx{\Pi}^*=[\vct{\Pi}_1^*,\dots,\vct{\Pi}_p^*]\in\mathbbm{R}^{d\times p}$ using the source data. It learns the columns in $\mtx{\Pi}^*$ such that they are orthogonal to each other. The formalization in \cite{chen2023project} is as follows:
%It learns the column in $\mtx{\Pi}^*$ with the constraint that it is orthogonal to all vectors before it 
\looseness=-1
\begin{align}
\label{eq: step_1}
    \vct{\Pi}_i^* = \arg\min_{\Pi_i} \hat{\E}_{(\vct{z}, y)\in\mathcal{E}^s} l( \vct{\Pi}_i^\top \vct{z}, y)\\
    \nonumber
    ~~s.t.~~ \vct{\Pi}_i^* \perp \vct{\Pi}_j^*\text{ for all } j<i,
\end{align}
where $l$ is the loss function. Then it performs linear probing on the projected embeddings of the target data
\begin{align}
\label{eq: step_2}
    \vct{w}_{\prot}^* = \arg\min_{\vct{w}} \hat{\E}_{(\vct{z}, y)\in\mathcal{E}^t} l( \vct{z}^\top \mtx{\Pi}^* \vct{w}, y  ).
\end{align}
Intuitively, in the source distribution, $\vct{v}_3$ does not carry any label information, thus it is not learned by \( \vct{\Pi}^* \) in Eq. \eqref{eq: step_1}. Thus, no information carried along \( \vct{v}_3 \) would exist after passing through \( \vct{\Pi} \). Therefore, any linear model on top of the projected embeddings would fail on any test example \( \vct{z} \) that contains \( y\vct{v}_3 \), which accounts for half of the target data.

In contrast, \ours\ can succeed because we directly train a linear model on the mixture of source and target embeddings. The linear model, having been exposed to \( \vct{v}_1, \vct{v}_2, \vct{v}_3 \),  will effectively learn all these features.

The following theorem summarizes the above discussion:
\begin{theorem}\label{thm: pro2_vs_ours}
Assuming that the noise is sufficiently small, \( \sigma = o(1) \), and \( l(.,.) \) is the MSE loss, then w.h.p.:\\
(1) The test loss on target data achieved by \prot\ can always be lower bounded by a constant order:
    \begin{align}
        \nonumber
        \E_{(\vct{z, y}) \in \text{target dist.} } l (\vct{z}^\top \mtx{\Pi}^*\vct{w}_{\prot}^*, y) \geq 0.5 - o(1). 
    \end{align}
(2) \ours\ learns \( \vct{w}_{\ours}^* = \arg\min_{\vct{w}} \hat{\E}_{ (\vct{z}, y)\in\embmx } \allowbreak l(\vct{z}^\top  \vct{w}, y) \), as described in Section \ref{sec: method}. When $\exists\epsilon=\Theta(1)$ s.t. $\epsilon<s<1-\epsilon$, it achieves a test loss on target data that can be upper bounded by:
    \begin{align}
        \nonumber
        &\E_{(\vct{z, y}) \in \text{target dist.} } l (\vct{z}^\top \vct{w}_{\ours}^*, y) \\
        \nonumber
        \leq& \bigg(\sigma + O(\sqrt{\frac{\log m}{m}})\bigg)^2 = o(1).
        \vspace{-.2cm}
    \end{align}
% \begin{itemize}
%     \itemsep0em
%     \item The test loss on target data achieved by \prot\ can always be lower bounded by a constant order:
%     \begin{align}
%         \nonumber
%         \E_{(\vct{z, y}) \in \text{target dist.} } l (\vct{z}^\top \mtx{\Pi}^*\vct{w}_{\pro}^*, y) \geq 0.5 - o(1). 
%     \end{align}
%     \vspace{-4mm}
%     \item \ours\ learns \( \vct{w}_{\ours}^* = \arg\min_{\vct{w}} \hat{\E}_{ (\vct{z}, y)\in\embmx } \allowbreak l(\vct{z}^\top  \vct{w}, y) \), as described in Section \ref{sec: method}. When $\exists\epsilon=\Theta(1)$ s.t. $\epsilon<s<1-\epsilon$, it achieves a test loss on target data that can be upper bounded by:
%     \begin{align}
%         \nonumber
%         &\E_{(\vct{z, y}) \in \text{target dist.} } l (\vct{z}^\top \vct{w}_{\ours}^*, y) \\
%         \nonumber
%         \leq& \bigg(\sigma + O(\sqrt{\frac{\log m}{m}})\bigg)^2 = o(1).
%     \end{align}
%     %which approaches \( 0 \) as \( m \) increases.
% \end{itemize}
\end{theorem}

\begin{figure}
    \centering
\includegraphics[width=.2\textwidth]{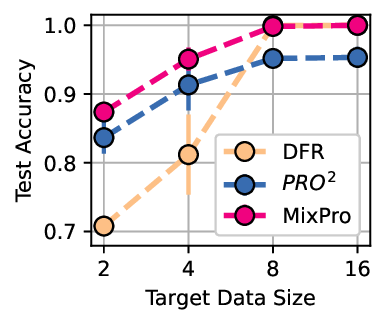}
\vspace{-.6cm}
    \caption{Comparison between methods on synthetic data.}
    \label{fig:synthetic_pro2}
\vspace{-.6cm}
\end{figure}

\textbf{The necessity of mixing with source data (advantage over DFR).}
Here, we further discuss why mixing target data with source data is important, compared to DFR, which solely trains on the target data. We provide a rough intuition by considering the extreme case where the noise is very large, \( \sigma = \omega(1) \). In this scenario, with probability of at least $1-\delta$, the test loss on the target distribution $\E_{(\vct{z, y}) \in \text{target dist.} } l (\vct{z}^\top \vct{w}_{\ours}^*, y) $, takes the form:
\begin{align}
    \nonumber
    (\psi(s, \sigma) +  O(s^2\sigma^2\sqrt{\frac{\log(1/\delta)}{m}}) )^2,
\end{align}
where $\psi(s, \sigma)$ is a function of $s$ and $\sigma$, independent of $m$ and $\delta$.
See explanation in Appendix \ref{apdx: explanation_large_sigma}.
%See exact function and discussion in \blue{Appendix \ref{apdx: explanation_large_sigma}}.
Note that using \( s=1 \) in our method corresponds to DFR, i.e., using only the target data. We observe that reducing \( s \), or placing more weight on the source data, diminishes the second term, which reflects the error due to the interaction of a small sample size $m$ and noise $\sigma$. Next, we will consider another example that allows us to provide a more fine-grained analysis, illustrating the trade-off involved in selecting the value of \( s \). In Fig. \ref{fig:synthetic_pro2}, we compare \ours\ with DFR and \prot\ using synthetic data similar to that in our analysis (see details in Appendix \ref{apdx: synthetic}).
% \begin{align}
%     \nonumber
%    \E_{(\vct{z, y}) \in \text{target dist.} } l (\vct{z}^\top \vct{w}_{\ours}^*, y)\leq  O( \sigma s^2 \sqrt{\frac{\log 1/\delta}{m}}).
% \end{align}

\subsection{A case study of subpopulation shift: %finegrained analysis on the 
trade-off between learning target and overfitting noise}\label{sec: theory_subpop}

Here, we adopt a variation of the subpopulation shift
%\blue{/spurious correlation} 
model used in prior work \cite{sagawa2020investigation,wald2021calibration,aubin2021linear,yao2022improving,xue2023understanding} to provide a fine-grained analysis demonstrating the benefits of \ours\ over DFR, which only uses target data. We will show that \ours\ makes a trade-off between learning from the target data %information 
and overfitting noise, outperforming both DFR  and only training on the source data.%\looseness=-1

We first define a family of data distributions \( \mathscr{E}(\mu) \) parameterized by \( \mu \), where each embedding-label pair \( (\vct{z}, y) \) is generated as follows. First, sample the label \( y \) uniformly from \( \{1, -1\} \), then sample \( a \) from \( \{1, -1\} \) with probabilities \( \Pr(a=y) = \mu \) and \( \Pr(a=-y) = 1-\mu \). Finally, generate the embedding as \( \vct{z}=[z_1, z_2, \vct{\xi}^\top]^T \) in \( \mathbbm{R}^d \), where \( z_1\sim\mathcal{N}(y, \sigma_1^2) \), \( z_2\sim\mathcal{N}(a, \sigma_2^2) \), and \( \vct{\xi} \sim\mathcal{N}(0, \frac{\sigma_\xi^2}{d-2}\mtx{I}_{d-2}) \). This model and its variations, widely utilized in many prior works \cite{sagawa2020investigation,wald2021calibration,aubin2021linear,yao2022improving,xue2023understanding}, are designed to simulate a specific type of distribution shift known as subpopulation shift or spurious correlation. The first coordinate \( z_1 \)
, the \textit{core} feature, 
carries the label information, while the second coordinate $z_2$, the \textit{spurious} feature, carries information about an attribute \( a \), with its correlation to the label \( y \) is dictated by \( \mu \). When \( \mu > 1/2 \), \( a \) and \( y \) are correlated. The remaining coordinates represent random noise. 
%This setting is observed in real datasets like Waterbirds and CelebA \blue{[cite]}, that contain a source distribution with a high \( \mu \) value, in contrast to a lower \( \mu \) value in the target distribution. Under these conditions, models tend to base predictions of \( y \) on \( a \), which impedes generalization to the target distribution.

Specifically, we consider the case where the source distribution is \( \mathscr{E}(p_\spu) \) with \( p_\spu > 1/2 \), and the target distribution is \( \mathscr{E}(1/2) \) where \( a \) is completely independent of \( y \).

\begin{figure}
\centering
\vspace{-2mm}
\subfigure[\small s=0(only source)\label{subfig: s_0}\looseness=-1]{
	\includegraphics[width=0.14\textwidth, valign=t]{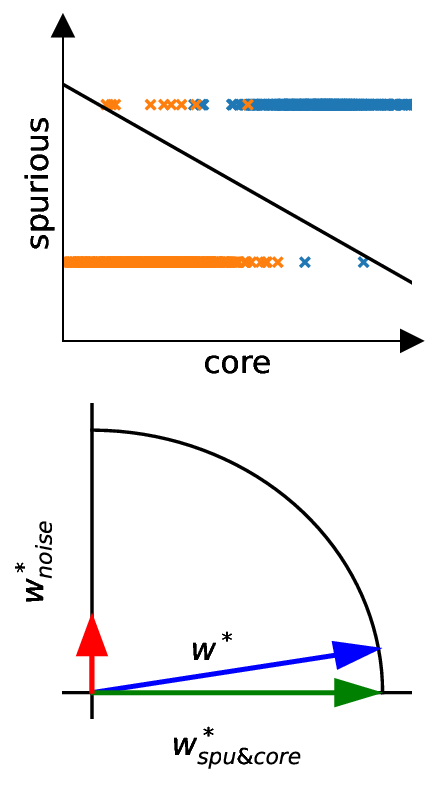}
	}
\subfigure[$s=0.4$  \label{subfig: s_0.4}]{
	\includegraphics[width=0.14\textwidth, valign=t]{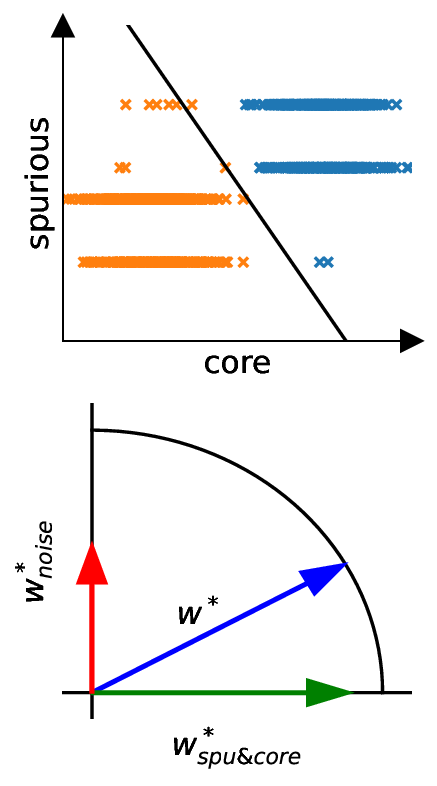}
	}
\subfigure[$s=1$ (DFR)  \label{subfig: s_1}]{
	\includegraphics[width=0.14\textwidth, valign=t]{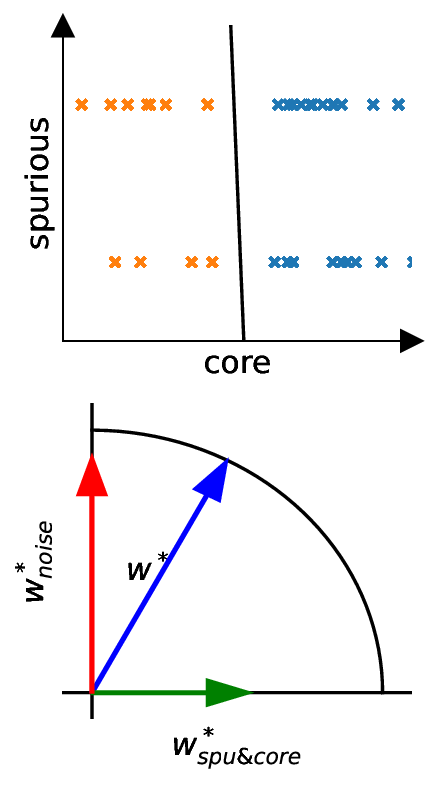}
	}
\vspace{-2mm}
\caption{\textbf{A trade-off between 
learning target information and preventing overfitting noise.} \textbf{Top: incorporating more target information with increasing $s$.
} The orange and blue dots represent examples from two classes 
in the spurious and core coordinates. The black line represents the decision boundary of the learned model along these first two coordinates. 
\textbf{Bottom: With increasing $s$, the model has a larger component in the coordinates for random noise.} 
The normalized weights vector is displayed in blue. Its component in the first two coordinates is shown in green, and the component in the directions of noise is shown in red.
}
\label{fig: visualization}
\vspace{-3mm}
\end{figure}

\textbf{The large source data is effectively less noisy than the small target data.} %To understand why, 
Consider the very small sample size of the small target data. When the dimension is significantly larger than the number of examples, there is a high probability that %the noise is dispersed. Specifically, 
the noise in each target example is asymptotically orthogonal to that in every other example, according to the concentration behavior of random Gaussian vectors. Due to this orthogonality, the noise appear unique to each example and thus the model can %leverage the noise to memorize each label, 
learn the noise to predic labels, especially when the scale of the noise is significantly large. 
%%%
Similarly, in real world data, %this analogously holds in the sense that 
noises are spread out across the entire embeddings with low correlation among themselves, as opposed to the useful features, which are usually concentrated in a lower-dimensional subspace \cite{papyan2017convolutional,morcos2018importance,oymak2019generalization,huh2021low}.
%%%
In contrast, in the source data, which comprises many more examples, %we can observe the full extent of the noise. It becomes evident that 
the noise %appears arbitrarily, regardless of the label. 
cannot be used for prediction of the label.
Therefore, the small target data is more affected by noise, while the source data is less so. %\looseness=-1

\begin{figure}[!t]
    \centering
    \subfigure[vary severity of shift\label{subfig: vary_p}]{\includegraphics[width=.2\textwidth]{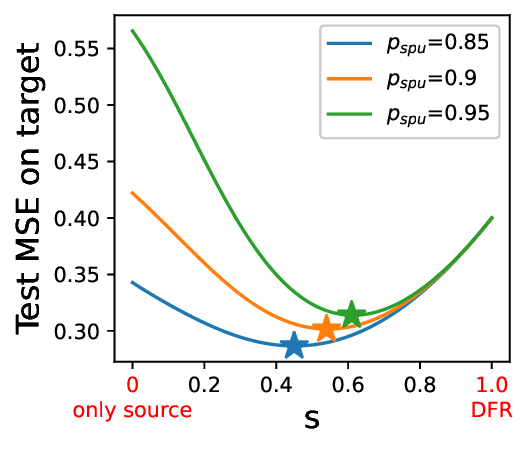}
    }
    \subfigure[vary $\frac{\text{noise}}{\text{target data size}}$ \label{subfig: vary_r}]{\includegraphics[width=.2\textwidth]{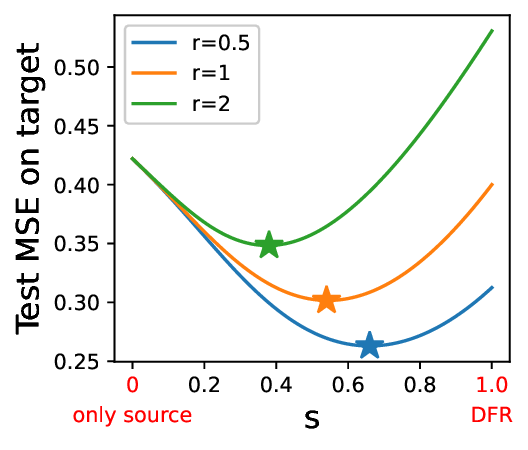}
    }
    \vspace{-2mm}
    \caption{ Left: A larger \( p_\spu \), indicating a more severe shift, necessitates a larger \( s \); Right: A larger \( r=\frac{\sigma_\xi^2}{m} \), signifying greater noise or a smaller target data size, necessitates a smaller \( s \).}
    \label{fig: subpopshift_curve}
    \vspace{-8mm}
\end{figure}

\textbf{Visualizing the trade-off between learning target information and preventing overfitting noise.} %To begin with an intuitive explanation of the tradeoff,
First, we visualize the decomposition of the weights \( \vct{w}^* \) of a linear model trained using \ours\ in Figure \ref{fig: visualization}. Specifically, we break down the weights into two components: \( \vct{w}^*_{s,c} \) and \( \vct{w}^*_{noise} \). The former is the component of \( \vct{w}^* \) in the span of the spurious and core features, while the latter is the component in the orthogonal complement, corresponding to noise. From Figure \ref{fig: visualization}, we observe that as \( s \) increases, the data distribution leans more towards the target distribution, leading the model to align more with the core coordinate (top row). However, simultaneously, the model also learns more noise (bottom row). Fig. \ref{subfig: s_1} shows that at \( s=1 \) (using only target data, i.e., DFR \cite{kirichenko2022last}), the model learns excessive noise. In contrast, Fig. \ref{subfig: s_0} shows that at \( s=0 \) (using only source data), the model learns predominantly from the spurious feature. As shown in Fig. \ref{subfig: s_0.4}, by setting an intermediate \( s \), we can strike a balance between learning target information and avoiding overfitting noise.

\textbf{Formal theoretical results.} Next, we present our formal analysis. We make the following assumptions:
\begin{assumption}\label{assump: asymptotic}
(1) Similar to \cite{sagawa2020investigation}, we assume that the second coordinate, carrying the spurious attribute \( a \), has a smaller variance than the first coordinate, which carries the true label information. This makes the spurious attribute easier to learn. Specifically, we set their corresponding variances to be \( \sigma_1 > 0 \) and \( \sigma_2 = 0 \). (2) To obtain closed-form results, we examine the high-dimensional asymptotic limit where \( n \), \( d \), and \( m \) tend to infinity. To account for the small size of the target data, we assume that \( \frac{n}{d} \rightarrow \infty \), while \( \frac{m}{d} \rightarrow 0 \), ensuring that \( m \) is much smaller than \( n \). (3) Additionally, we assume that \( \sigma_\xi \rightarrow \infty \), while  ratio between the noise and the sample size of the target data, i.e., $\frac{\sigma_\xi^2}{m}$, remains constant at $r$. This allows us to observe the effects of the target dataset's size and noise level.
\end{assumption}

% We first show that increasing $s$ in \ours (putting more weight on target data) leads to more overfitting to noise, which is why DFR ($s=1$) overfits the most noise.
% \begin{theorem}
% The linear model $\vct{w}_{\ours}^*$ learned by \ours with $l()$ being MSE loss with $\ell_2$ regularization has the following alignment with noise, measured by the norm of the model weights except the first two coordinates.
% \begin{align}
%     \nonumber
%  \|\vct{w}_{\ours}^*[3:]\|   \propto\sqrt{\mathbf{Var}[h_i]},
% \end{align}
% Actually this does not work
% \end{theorem}
% This is also illustrated in Figure XXXX. Now that we see \ours is less prone to noise than using only target data (DFR), the remaining part is intuitive: \ours is also better than using only source data because the mixed embeddings contain additional information about the target data. See the illustration in Fig XXX. In the next theorem, 

We perform an asymptotic analysis to derive the exact expression of the test performance on the target data, and plot it to observe the benefit of mixing source and target data, as well as how it interacts with the severity of distribution shift, noise level, and target dataset size.
\begin{theorem}\label{thm: subpop}
Consider the linear model $\vct{w}_{\ours}^*$ learned by \ours\ with $l(.,.)$ being MSE loss with $\ell_2$ regularization. With assumption \ref{assump: asymptotic}, the test loss on the target distribution achieved by  $\vct{w}_{\ours}^*$ can be expressed in closed form in terms of $p_\spu, \sigma_1, r$ and $s$, as  shown in \blue{Appendix \ref{apdx: subpop}}.
\end{theorem}
\vspace{-2mm}
We plot the expression in Figure \ref{fig: subpopshift_curve}. Note that \( s=0 \) corresponds to using only source data and \( s=1 \) to using only target data (DFR). Across all plots, we observe that the test loss on the target distribution first decreases and then increases as \( s \) increases. An intermediate value of \( s \) yields the best performance, confirming that a mix between the two performs better than using either only source or only target (DFR). In Fig \ref{subfig: vary_p}, we plot the curves under different \( p_\spu \) values, which indicate the severity of the distribution shift. We see that with a larger \( p_\spu \), implying a greater shift, the optimal \( s \) becomes larger, suggesting more emphasis on the target data. In Fig \ref{subfig: vary_r}, we plot the curves under different \( r \) values, the ratio of noise to target sample size. We observe that a larger \( r \), indicating either greater noise or a smaller number of target examples, results in the optimal \( s \) being smaller, aligning with the intuition that in such cases, we should rely more on the source data to counteract noise. This is also empirically confirmed on real datasets in Figure \ref{fig: s_size} in Appendix \ref{apdx: ablation}.

\section{Experiments}\label{sec: exp}

In this section, we will demonstrate the superior performance of \ours\ over existing baselines across various datasets and using different backbone models. Additionally, we will explore realistic hyperparameter tuning in the few-shot scenario, which is not addressed in prior work.

%Additionally, we will explore the effect of the hyperparameter \( s \), illustrating the tradeoff involved in its selection.

% \begin{figure*}[!t]
%     \centering
%     \includegraphics[width=.7\textwidth]{figures/experiments/all_resnet.eps}
%     \caption{ResNet ImageNet 1k, finetuned on source data}
%     \label{fig:resnet}
% \end{figure*}

% \begin{figure*}[!t]
%     \centering
%     \includ
% egraphics[width=.7\textwidth]{figures/experiments/all_swagvitl16.eps}
%     \caption{SWAGVITL16}
%     \label{fig:swagvitl16}
% \end{figure*}

\begin{figure*}[!t]
    \centering
    \includegraphics[width=.85\textwidth]{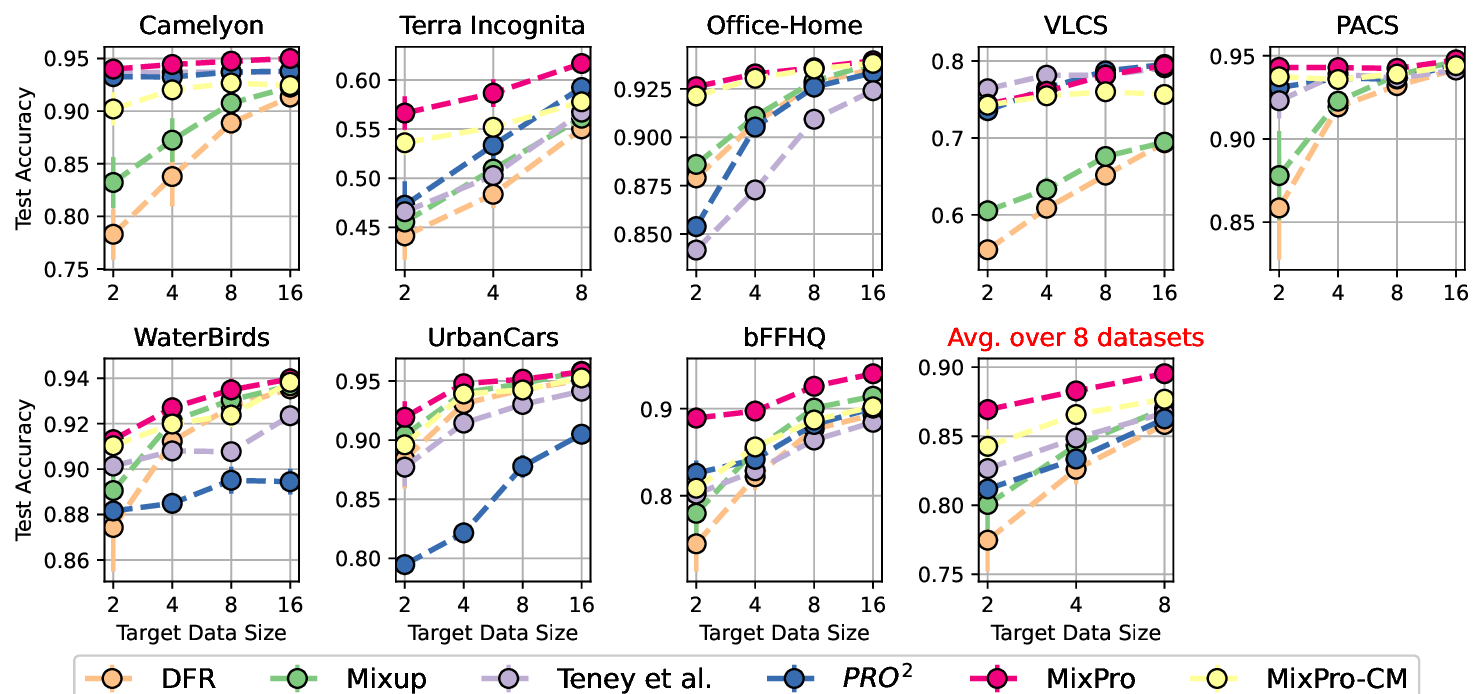}
    \vspace{-2mm}
    \caption{Test accuracy on the target distribution versus target data size for all baselines across the 8 datasets we consider. Here, we use the SWAG-pretrained ViT-L/16 model as the backbone . Overall, while some methods may perform comparably to \ours\ on certain datasets, they falter on others. In contrast, \ours\ consistently achieves the best performance across datasets and data sizes.}
\label{fig:large_val_swagvit}
\vspace{-4mm}
\end{figure*}

% \begin{figure*}[!t]
%     \centering
%     \includegraphics[width=.7\textwidth]{figures/experiments/large_val/all_resnet.eps}
%     \caption{ResNet ImageNet 1k, finetuned on source data}
%     \label{fig:large_val_resnet}
% \end{figure*}

\begin{figure*}[!t]
    \centering
    \includegraphics[width=.98\textwidth]{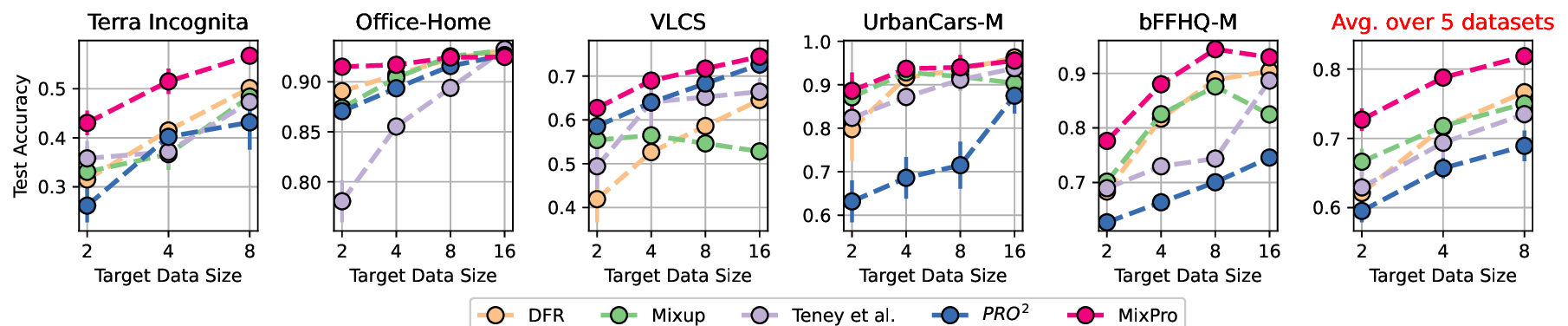}
    \vspace{-2mm}
    \caption{ Results with hyperparameters tuned using cross-validation with only the few given target data. We use the SWAG-pretrained ViT-L/16 model as the backbone. In these scenarios, \ours\ continues to outperform the others as the best method. }
    \label{fig:cross_val_swagvit}
    \vspace{-3mm}
\end{figure*}

% \subsection{Settings}

\textbf{Datasets.} We consider 8 datasets with different types of distribution shifts. These include 3 subpopulation shift benchmarks: WaterBirds \cite{sagawa2019distributionally}, bFFHQ \cite{kim2021biaswap}, and UrbanCars \cite{li2023whac}; along with 5 domain generalization benchmarks: PACS \cite{li2017deeper}, VLCS \cite{fang2013unbiased}, Office-Home \cite{venkateswara2017deep}, Terra Incognita \cite{beery2018recognition}, and Camelyon17 \cite{koh2021wilds}. Due to space limitations, we defer the details to Appendices \ref{apdx: details_data} and \ref{apdx: details_dist}. For the target dataset size, we consider the few-shot scenario where \{2, 4, 8, 16\} data points per class are sampled from the target distribution {(except for Terra Incognita, for which we use only \{2, 4, 8\} samples, because the smallest class, i.e., class index 9, has only 12 examples)}. The test set is constructed using the remaining data in the target distribution.
%\looseness=-1

% WaterBirds \cite{sagawa2019distributionally}, bFFHQ \cite{kim2021biaswap} and UrbanCars \cite{li2023whac} exhibit subpopulation shifts. In these datasets, certain subpopulations are significantly underrepresented in the source distribution. Following \cite{chen2023project}, we examine two types of target distribution that differ from the source distribution: (1) Minority (M), which only contains the minority subpopulations, and (2) Balanced (B), which contains an equal number of examples from each subpopulation. The details are deferred to Appendix \blue{todo}. PACS, VLCS, Office-Home and Terra Incognita are domain generalization datasets, each comprising data collected from multiple distinct environments/domains. Following prior work \cite{gulrajani2020search,li2017deeper,fang2013unbiased,beery2018recognition,venkateswara2017deep}, we select one domain as the target distribution while letting the remaining domains form the source distribution. Camelyon17 \cite{koh2021wilds} is also a domain generalization dataset, where the source and target distributions consist of data collected from different hospitals. 
% For the target dataset size, we consider the few-shot scenario where a range of \{2, 4, 8, 16\} data points per class are sampled from the target distribution. The test set is constructed using the remaining data in the target distribution.

\textbf{Baselines.} 
%\yx{todo: potentially add baselines that do not use labels. e.g., D-BAT or DivDis, if have time} 
We consider the following four baselines: (1) DFR (Deep Feature Reweighting) \cite{kirichenko2022last}, which performs standard linear probing using the target data. (2) (Manifold) Mixup \cite{zhang2017mixup,verma2019manifold} where we additionally apply Mixup to the embeddings while performing DFR.   (3) \citet{teney2022evading}, which trains multiple models on the source data while minimizing the similarity between these models' gradients. We note that their method does not specify how to apply these models to target data. %, and any approach can be used. 
To maximize effectiveness, we train a linear model on the concatenation of the outputs of these models using the target data. The final prediction is thus a weighted combination of these models, with weights determined using the given target data.
(4) \prot\ \cite{chen2023project}, which first learns a linear projection to map the embeddings onto orthogonal directions using source data, and then performs linear probing using target data on the projected embeddings. \looseness=-1

\textbf{Backbone models.} For the backbone model that produces embeddings, we consider two models: (1) The standard ImageNet-pretrained ResNet 50 and (2) the ViT-L/16 model pretrained with SWAG \cite{singh2022revisiting}. These models are publicly available in TorchVision. For model (1), we additionally fine-tune the backbone model on the source data following \cite{kirichenko2022last, izmailov2022feature}, as \cite{rosenfeld2022domain} suggests that this significantly improves the final performance. For model (2), we use the original weights directly, as \cite{mehta2022you} has demonstrated the superior robustness provided by the original SWAG weights.\looseness=-1

\textbf{Hyperparameter tuning.}  In Table \ref{tab: hp}, we present the range of hyperparameters considered for each method, and the ranges are sourced from the corresponding original papers. We test our method and existing baselines using two different approaches to  tune hyperparameters, for a comprehensive evaluation. The first setting (Section \ref{sec: large_val}) follows the practice of prior works \cite{teney2022evading,kirichenko2022last,izmailov2022feature,chen2023project}, where hyperparameters are chosen based on performance on a held-out large validation set from the target distribution. We note that such a large validation set is essentially \textit{hypothetical} and is not typically available in real-world scenarios, particularly in a few-shot setting where only a limited number of target data points are accessible. Prior work that uses this approach does admit that selecting hyperparameters in a realistic manner remains an open problem that they have not addressed yet. Therefore, we also test a second, more realistic setting for hyperparameter selection in few-shot learning, where we tune using cross-validation with only the few available target data. The details are provided in Section \ref{sec: cross_val}.\looseness=-1

\subsection{Results with hypothetical large validation data} \label{sec: large_val}

The results for SWAG-pretrained ViT-L/16 are displayed in Figures \ref{fig:large_val_swagvit}, and results for ResNet50 are deferred to Fig \ref{fig:large_val_resnet} in Appendix \ref{apdx: rn50}. We report the results over 5 runs. Overall, while some methods may perform comparably to \ours\ on certain datasets, they falter on others. In contrast, \ours\ consistently achieves the best performance across  various datasets, data sizes, and backbone models, as shown from the \textit{average} plot in Figures \ref{fig:large_val_swagvit} and \ref{fig:large_val_resnet}. %Below is a more detailed analysis of the results. 
%%%
The superiority of \ours\ over DFR and Mixup is particularly noticeable in scenarios with small target dataset sizes (e.g., 2 per class). This is because DFR relies solely on the limited target data, leading to an insufficient sample size and consequently a higher risk of overfitting to noise, as elaborated in Section \ref{sec: theory}.  The performance of \prot\ largely depends on the severity of the distribution shift. Specifically, based on our fine-grained results in Appendix \ref{apdx: exp_subpop_all} for the subpopulation shift datasets (WaterBirds, UrbanCars, bFFHQ), \prot\ works better when the target distribution is \textit{balanced} (Fig \ref{fig:subpop_all} top), i.e., all subpopulations are equally represented in the target distribution, despite still being outperformed by \ours. However, in scenarios where the target distribution only consists of the \textit{minority} subpopulations, indicating a more severe shift, \prot\ achieves worse performance compared to others (Fig \ref{fig:subpop_all} middle). The corresponding three plots in Fig \ref{fig:large_val_swagvit}  present the average results over the two types of shifts. Furthermore, on domain generalization datasets (Fig \ref{fig:large_val_swagvit}, top row), which present a severe shift as the target and source come from completely disjoint domains, \prot\ also achieves less satisfactory performance overall. In contrast, \ours\ consistently performs well across datasets. This is supported by our theoretical analysis in Section \ref{sec: theory_domain_gen}, which explains that when the source and target differ significantly, such as when the target has examples not present in the source's support, \prot\ may filter out crucial features of the embeddings during projection, while \ours\ learns all vital aspects, thereby confirming \ours's advantage over \prot. Besides, \citet{teney2022evading} is also surpassed by \ours, for a similar reason; they initially train multiple models exclusively on source data before applying them to the target data, potentially missing relevant features. % relevant to the target distribution. 
\vspace{-1mm}
\subsection{Results with cross validation using few-shot data}\label{sec: cross_val}
\vspace{-1mm}
Prior works, such as \cite{kirichenko2022last, izmailov2022feature, chen2023project}, tuning hyperparameter---including learning rate, weight decay, and method-specific hyperparameters---by evaluating accuracy on a large held-out validation set. However, this may not be realistic for few-shot adaptation, as it contradicts the assumption of having only a few target data points available. Therefore, to evaluate if the methods can operate effectively in a true few-shot scenario without additional data, we employ standard $k$-fold cross-validation using the limited target data available for hyperparameter selection. Considering the smallest case in our experiments, where the target data size is only 4 (2 per class and 2 classes), we set \( k=2 \) to ensure that each fold has at least one data point per class. 

We conduct experiments on Terra Incognita, VLCS, UrbanCars, bFFHQ and Offie-Home. 
%The results are displayed in 
Figs \ref{fig:cross_val_swagvit} and \ref{fig:cross_val_resnet} show the results for different model architectures. We observe that although the overall performance slightly drops, it remains reasonable, as the decrease is mostly within 10\% compared to the case where a large validation set is used, considering that no extra information beyond the given few examples is used. This confirms that cross-validation strategy is promising and is a potential solution to the hyperparameter selection. Importantly, \ours\ still stands out with the highest performance, showing that its advantage is retained even when hyperparameters are selected using very limited information.  \looseness=-1

\subsection{Relaxing the requirement of source data availability}

In some scenarios, the source and target data may not be available simultaneously. In other words, by the time we encounter the target data and wish to adapt to it, the source data may no longer be accessible. However, \ours\ requires both data to be available at the same time to train a linear probe on the mixed source and target embeddings. This implies that one may always need to store the source embeddings to use them when a new target data is encountered. To alleviate this requirement, we propose a simple approach of storing only minimal information about the source data, thereby making the method more flexible.

Specifically, we store only the average embeddings (mean) of each class, resulting in only \#classes embeddings. When we create the mixed embeddings, we combine each target data point with the corresponding class mean from the source. We call this modified method \ours-CM and test it on all the datasets under the same settings as shown in Figure \ref{fig:large_val_swagvit}. From the figure, we observe that \ours-CM performs slightly worse than the original \ours, but still outperforms all other methods. We believe that further exploration into how to improve it, such as retaining more fine-grained information than just class means, would be a valuable direction for future research.

% \subsection{Mixing weight and target dataset size}\label{sec: exp_different_s}

% Despite showing in the previous section that simply performing cross-validation with the available data can already yield good hyperparameters leading to strong performance, we still aim to offer some general guidance on how the size of the target data influences the choice of \( s \).  As depicted in Fig \ref{subfig: vary_r}, our theoretical analysis in Section \ref{sec: theory_subpop} has already indicated that with fewer target data points, a smaller \( s \) is preferable.  This means placing more emphasis on the source data to counteract the noise.  We confirm this intuition on real datasets. In Figure \ref{fig: s-size}, we present the results on Terra Incognita and XXX, and indeed observe the aforementioned trend.\looseness=-1

% \begin{figure}
%     \centering
%     \includegraphics[width=.23\textwidth]{figures/experiments/TI_TI_L46rn50_confident_False.eps}
%     \includegraphics[width=.2\textwidth]{example-image-a}
%     \caption{The mixing weight $s$ and the size of the target dataset. Terra Incognita.  [The figure is not completed yet; just an exampple]. }
%     \label{fig: s-size}
% \end{figure}

\vspace{-1mm}\section{Conclusion}\vspace{-1mm}
In this work, we propose \ours\ for few-shot adaptation to distribution shifts. \ours\ performs linear probing on mixed embeddings of source and target data,  avoiding overfitting to the small target data while learning target information. We provide theoretical analysis showing its advantages over previous methods and conduct extensive experiments demonstrating its superior performance across various scenarios.
\looseness=-1

\section*{Acknowledgment.} This research was supported by the National Science Foundation CAREER Award 2146492 and Cisco Systems.

\bibliography{reference}
\bibliographystyle{icml2024}

\newpage
\onecolumn
\appendix
\section{Theoretical Analysis}

\subsection{Discussion on \citet{teney2022evading}}\label{apdx: teney}

\cite{teney2022evading} proposes training multiple models ($h_1, \dots, h_{n_{\text{models}}}$) on the embeddings of the source data while minimizing the similarity of these models, before adapting to the target data. We will see that this process closely resembles the projection step of \prot. The objective of \cite{teney2022evading} is defined as:
\begin{align}
    \nonumber
    \min \left(\sum_{i=1}^{n_{\text{models}}}\hat{\E}_{(\vct{z}, y)\in\mathcal{E}^s}l(h_i(\vct{z}), y) + \lambda \sum_{i\neq j}\sum_{(\vct{z}, y)\in\mathcal{E}^s} \text{sim}_{\vct{z}}( h_i, h_j ) \right),
\end{align}
where $\text{sim}_{\vct{z}}( h_i, h_j )$ is the inner product between the gradients of two models with respect to the input $\vct{z}$.

Considering that $h$'s are linear models defined by $h_i(\vct{z}) = \vct{\theta}_i^\top \vct{z}$, the gradient with respect to any input $\vct{z}$ is just the model weights, $\nabla_{\vct{z}} h_i(\vct{z}) = \vct{\theta}_i$. Consequently, the objective becomes:
\begin{align}
    \nonumber
    \min_{\vct{\theta}_1, \dots, \vct{\theta}_{n_\text{models}}} \left(\sum_{i=1}^{n_{\text{models}}}\hat{\E}_{(\vct{z}, y)\in\mathcal{E}^s}l(\vct{\theta}_i^\top\vct{z}, y) + \lambda' \sum_{i\neq j} \vct{\theta}_i^\top \vct{\theta}_j \right).
\end{align}
This goal bears a strong resemblance to that of \prot\, where $\vct{\theta}_i$'s can be considered analogous to the projection vectors $\vct{\Pi}_i$'s in \prot\. The objective above is to align each projection's output with the label $y$, while promoting orthogonality among the projections. The only difference is that \cite{teney2022evading} achieves this through a regularization term, whereas \prot\ enforces exact orthogonality. We hypothesize that this additional flexibility is likely why \cite{teney2022evading} generally outperforms \prot\ with hyperparameter tuning in our experiments in Section \ref{sec: exp}.

\subsection{Proof of Theorem \ref{thm: pro2_vs_ours}}
\subsubsection{Analysis of \ours}\label{apdx: domaingen_ours}

\textbf{Notations.} Let $\{\tilde{\vct{z}}_i\}_{i=1}^n$ denote the mixed embeddings where $\tilde{\vct{z}}_i = (1-s)\vct{z}_i^s + s\vct{z}_{j_i}^t $. 

\textbf{Assumptions.} As mentioned in Section \ref{sec: theory_domain_gen}, we assume $n \rightarrow \infty$ and $\sigma=o(1)$. We also assume that $\exists\epsilon=\Theta(1)$ s.t. $\epsilon<s<1-\epsilon$ as mentioned in the statement of Theorem \ref{thm: pro2_vs_ours}.

Without loss of generality, we consider $\vct{v}_1, \vct{v}_2, \vct{v}_3$ to be $\vct{e}_1, \vct{e}_2, \vct{e}_3$, respectively, i.e., the three standard basis vectors.

By the closed form expression of the minimizer of the MSE loss, we can write the weight of the linear model found in \ours\ as
\begin{align}
\label{eq: closed_form}
    \vct{w}_{\ours}^* =  \hat{\mtx{\Sigma}}_\mixed^{-1} \hat{\vct{u}}_\mixed,
\end{align}
where 
\begin{align}
    \hat{\mtx{\Sigma}}_\mixed \coloneqq \frac{1}{n}\sum \tilde{\vct{z}}_i\tilde{\vct{z}}_i^\top,
\end{align}
and 
\begin{align}
    \nonumber
   \hat{\vct{u}}_\mixed \coloneqq & \frac{1}{n}\sum \tilde{\vct{z}}_i \tilde{y}_i.
\end{align}
Next, we analyze $\hat{\mtx{\Sigma}}_\mixed^{-1} $ and $\hat{\vct{u}}_\mixed$.

Firstly, we decompose $\hat{\mtx{\Sigma}}_\mixed $ as
\begin{align}
\hat{\mtx{\Sigma}}_\mixed = &
    (1-s)^2\frac{1}{n}\sum \vct{z}_i^{s}\vct{z}_i^{s \top}+s^2 \frac{1}{n}\sum \vct{z}_{j_i}^t \vct{z}_{j_i}^{t\top} +(1-s)s\frac{1}{n}\sum\vct{z}_i^{s}\vct{z}_{j_i}^{t\top} + (1-s)s\frac{1}{n}\sum\vct{z}_{j_i}^{t}\vct{z}_i^{s\top} \\
    = & \mtx{\Sigma}_\mixed + \Delta \mtx{\Sigma}_\mixed,
\end{align}
where
\begin{align}
    \nonumber
    \mtx{\Sigma}_\mixed \coloneqq \E \hat{\mtx{\Sigma}}_\mixed =  \mtx{\Sigma}_0 + \mtx{A}
\end{align}
with
\begin{align}
    \mtx{\Sigma}_0 =  \begin{bmatrix}
        \frac{1}{2}(1-s)^2 & \frac{1}{4}(1-s)s & \frac{1}{4} (1-s)s  \\
        \frac{1}{4} (1-s)s & \frac{1}{4}s^2 +\frac{1}{4}(1-s)^2 +\frac{1}{4} & \frac{1}{4}(1-s)s  \\
        \frac{1}{4} (1-s)s & \frac{1}{4} (1-s)s & \frac{1}{2}s^2  \\ 
    \end{bmatrix}
\end{align}
and
\begin{align}
    \mtx{A} = \begin{bmatrix}
        s^2 \sigma^2 & 0 & 0 \\
        0 & 0 & 0\\
        0 & 0 & (1-s)^2 \sigma^2
    \end{bmatrix}.
\end{align}

% Define
% \begin{align}
%     \mtx{\Sigma_0} \coloneqq \begin{bmatrix}
%     \mtx{A} & \mtx{0}_{3\times (d-3)}\\
%     \mtx{0}_{(d-3)\times 3} & ((1-s)^2+s^2)\frac{\sigma_\xi^2}{d-2} \mtx{I}_{(d-3)\times (d-3)}
% \end{bmatrix}
% \end{align}

% \begin{align}
%     \nonumber
%     \mtx{\Sigma}_\mixed \coloneqq \E \hat{\mtx{\Sigma}}_\mixed =  
% \begin{bmatrix}
%     \mtx{A} + \mtx{A}' & \mtx{0}_{3\times (d-3)}\\
%     \mtx{0}_{(d-3)\times 3} & ((1-s)^2+s^2)\frac{\sigma_\xi^2}{d-2} \mtx{I}_{(d-3)\times (d-3)}
% \end{bmatrix}
% \end{align}

% \begin{align}
%     \mtx{A} = \begin{bmatrix}
%         \frac{1}{2}(1-s)^2 & \frac{1}{4}(1-s)s & \frac{1}{4} (1-s)s  \\
%         \frac{1}{4} (1-s)s & \frac{1}{4}s^2 +\frac{1}{4}(1-s)^2 +\frac{1}{4} & \frac{1}{4}(1-s)s  \\
%         \frac{1}{4} (1-s)s & \frac{1}{4} (1-s)s & \frac{1}{2}s^2  \\ 
%     \end{bmatrix}
% \end{align}

% \begin{align}
%     \mtx{A}' = \begin{bmatrix}
%         s^2 \frac{\sigma_\xi^2}{d-2} & 0 & 0 \\
%         0 & 0 & 0\\
%         0 & 0 & (1-s)^2 \frac{\sigma_\xi^2}{d-2}
%     \end{bmatrix}
% \end{align}

By applying concentration bounds for Gaussian vectors and matrices (Theorems 6.1 and 6.3 of \cite{wainwright2019high}) we can obtain the following. We omit the exact derivation, but the process is similar to, for example, Section B.2.2 in \cite{xue2023understanding}.
\begin{align}
\label{eq: delta_sigma}
    \|\Delta \mtx{\Sigma}_\mixed \|_2 
    %\leq & O( s^2\sigma_\xi^2\frac{1}{d}\sqrt{\frac{\log 1/\delta}{m}} + s(1-s) \frac{\sigma_\xi}{\sqrt{m}} \frac{\sqrt{\log1/\delta}}{\sqrt{d}})\\
    %\nonumber 
   \leq & O(s^2\sigma^2 \sqrt{\frac{\log1/\delta}{m}}) + O( s(1-s)\sigma \sqrt{\frac{\log1/\delta}{m}} ),
    % \nonumber
    % = & O( \sigma_\xi\sqrt{\frac{\log 1/\delta}{m}} + \sigma_\xi^2\sqrt{\frac{\log 1/\delta}{m}} ).  
    ~~~\text{with probability $\geq 1-\delta$}
\end{align}
Here, as we assume $\sigma = o(1)$, we further obtain $
   \|\Delta \mtx{\Sigma}_\mixed \|_2  = O( \sigma\sqrt{\frac{\log 1/\delta}{m}} )$.
Since $\exists\epsilon=\Theta(1)$ s.t. $\epsilon<s<1-\epsilon$, we have\looseness=-1
\begin{align}
\label{eq: sigma_0_O1}
    \|\mtx{\Sigma}_0\|_2 \leq O(1), \|\mtx{\Sigma}_0^{-1}\|_2 \leq O(1) 
\end{align}
By applying the classical result for the inverse of perturbation \cite{demmel1992componentwise} and combining it with \ref{eq: sigma_0_O1}, we obtain
\begin{align}
\label{eq: inverse}
    \|\hat{\mtx{\Sigma}}_\mixed^{-1} - \mtx{\Sigma}_0^{-1} \|_2 \leq   O(\|\mtx{A} + \Delta \mtx{\Sigma}_\mixed \|_2) = O(  \sigma\sqrt{\frac{\log 1/\delta}{m}} ) ~~~\text{with probability $\geq 1-\delta$}.
\end{align}
Next, we examine $ \hat{\vct{u}}_\mixed $, which can be further decomposed into
\begin{align}
    \nonumber
   \hat{\vct{u}}_\mixed \coloneqq & \frac{1}{n}\sum \tilde{\vct{z}}_i \tilde{y}_i \\
    \nonumber
    = &\vct{u}_\mixed + \Delta \vct{u}_\mixed,
\end{align}
where
\begin{align}
    \nonumber
    \vct{u}_\mixed \coloneqq \E \vct{u}_\mixed = 
    \begin{bmatrix}
        \frac{1}{2}(1-s)\\
        \frac{1}{2}\\
        \frac{1}{2}s\\
    \end{bmatrix}.
\end{align}
Similarly, we can also bound $\|\Delta \vct{u}_\mixed\|_2$ using Gaussian concentration bounds
\begin{align}
    \nonumber
    \|\Delta \vct{u}_\mixed\|_2 \leq O(\frac{\sigma}{\sqrt{m}} \sqrt{\log1/\delta}) ~~~\text{with probability $\geq 1-\delta$}.
\end{align}
$\mtx{\Sigma}_0^{-1}\vct{u}_\mixed$ can be easily calculated 
\begin{align}
\mtx{\Sigma}_0^{-1}\vct{u}_\mixed = 
\begin{bmatrix}
    1\\
    1\\
    1\\
\end{bmatrix}.
\end{align}
Now we bound the distance between $\mtx{\Sigma}_0^{-1}\vct{u}_\mixed $ and $\hat{\mtx{\Sigma}}_\mixed^{-1}\hat{\vct{u}}_\mixed $. Observe that
\begin{align}
    \nonumber
    \hat{\mtx{\Sigma}}_\mixed^{-1} \hat{\vct{u}}_\mixed = \bigg(\mtx{\Sigma}_0^{-1}-(\mtx{\Sigma}_0^{-1} - \hat{\mtx{\Sigma}}_\mixed^{-1}) 
  \bigg)\bigg(\vct{u}_\mixed
  - (\vct{u}_\mixed - \hat{\vct{u}}_\mixed)
  \bigg).
\end{align}
By Cauchy–Schwarz inequality and realizing that $\|\vct{u}_\mixed\|_2$ and $\| \mtx{\Sigma}_0^{-1} \|_2$ are $\Theta(1)$, we have
\begin{align}
\label{eq: diff_inverse}
    \|\hat{\mtx{\Sigma}}_\mixed^{-1} \hat{\vct{u}}_\mixed - \mtx{\Sigma}_0^{-1}\vct{u}_\mixed\|_2 \leq & O( \| \hat{\mtx{\Sigma}}_\mixed^{-1} - \mtx{\Sigma}_0^{-1} \|_2 ) + O( \| \hat{\vct{u}}_\mixed - \vct{u}_\mixed \|_2 )\\
    \nonumber
    = & O(\sigma\sqrt{\frac{\log(1/\delta)}{m}}) ~~~\text{with probability $\geq 1-\delta$}.
\end{align}
Now we are ready to bound the test loss. We first define the covariance matrix of the target embeddings
\begin{align}
    \nonumber
    \mtx{\Sigma_t} = 
    \begin{bmatrix}
        \sigma^2 & 0 & 0 \\
        0 & \frac{1}{2} & 0\\
        0 & 0 & \frac{1}{2}
    \end{bmatrix}.
\end{align}
By deriving the test loss, leveraging the relation $y=\vct{\beta}_t^\top\vct{z}$ where $\vct{\beta}_t=[0, 1, 1]^\top$ in the target distribution, and 
incorporating Equation \ref{eq: closed_form}, we express the test loss as
\begin{align}
    \nonumber
   \E_{(\vct{z, y}) \in \text{target dist.} } l (\vct{z}^\top \vct{w}_{\ours}^*, y) =  \| \mtx{\Sigma_t}^{1/2}( \vct{\beta}_t - \hat{\mtx{\Sigma}}_\mixed^{-1} \hat{\vct{u}}_\mixed)  \|^2.
\end{align}
Applying Cauchy–Schwarz inequality and incoporating Equation \ref{eq: diff_inverse}  yield
\begin{align}
    \nonumber
    \sqrt{\E_{(\vct{z, y}) \in \text{target dist.} } l (\vct{z}^\top \vct{w}_{\ours}^*, y)} \leq & \| \mtx{\Sigma_t}^{1/2}( \vct{\beta}_t - \mtx{\Sigma}_0^{-1} \vct{u}_\mixed)  \| + \| \mtx{\Sigma_t}^{1/2}( \mtx{\Sigma}_0^{-1} \vct{u}_\mixed - \hat{\mtx{\Sigma}}_\mixed^{-1}\hat{\vct{u}}_\mixed )   \| \\
    \nonumber
    \leq & \sigma + O(\sqrt{\frac{\log(1/\delta)}{m}}).  
\end{align}
Therefore
\begin{align}
    \nonumber
   \E_{(\vct{z, y}) \in \text{target dist.} } l (\vct{z}^\top \vct{w}_{\ours}^*, y)
    %\leq \sigma_\xi^2 + O(\sigma_\xi\sqrt{\frac{\log(1/\delta)}{m}}).\\
    \leq & \bigg(\sigma + O(\sqrt{\frac{\log(1/\delta)}{m}})\bigg)^2.
\end{align}
Consequently, with a probability of at least $1-O(\frac{1}{\poly m})$
\begin{align}
    \nonumber
    \E_{(\vct{z, y}) \in \text{target dist.} } l (\vct{z}^\top \vct{w}_{\ours}^*, y)
    \leq & \bigg(\sigma + O(\sqrt{\frac{\log m}{m}})\bigg)^2 = o(1).
\end{align}

\subsubsection{Analysis of \prot }

We first solve $\mtx{\Pi}_1^*$, which is simply given by $\arg\min_{\vct{\Pi}_i} \hat{\E}_{(\vct{z}, y)\in\mathcal{E}^s} l( \vct{\Pi}_i^\top \vct{z}, y)$ without any restriction. Since we use MSE loss, the closed-form solution is given by 
\[
    (\hat{\E}_{(\vct{z}, y)\in\mathcal{E}^s} \vct{z}\vct{z}^\top )^{-1} \hat{\E}_{(\vct{z}, y)\in\mathcal{E}^s}y\vct{z}
\]
As we assume $n\rightarrow \infty$, the empirical expectation equals the true expectation, thus we can calculate that 
\[
    \hat{\E}_{(\vct{z}, y)\in\mathcal{E}^s} \vct{z}\vct{z}^\top =
    \begin{bmatrix}
        1/2 & 0 & 0\\
        0 & 1/2 & 0 \\
        0 & 0 & \sigma_\xi^2
    \end{bmatrix},
\] and 
\[
    \hat{\E}_{(\vct{z}, y)\in\mathcal{E}^s}y\vct{z} =
    \begin{bmatrix}
        1/2 \\
        1/2  \\
        0
    \end{bmatrix}.
\]
Therefore, we obtain that
\[
    \vct{\Pi}_1^* = 
    \begin{bmatrix}
        1\\
        1\\
        0
    \end{bmatrix}.
\]
For the remaining $\vct{\Pi}_i^*$'s ($i\geq 2$), we first write the loss function as follows
%rewrite the loss function by leveraging the relation $y=\vct{\beta}_s^\top \vct{z}$ where $\vct{\beta}_s=[1,1,0]^\top$ in the source distribution
\begin{align}
    \nonumber
    \hat{\E}_{(\vct{z}, y)\in\mathcal{E}^s} l( \vct{\Pi}_i^\top \vct{z}, y) = & \E_{(\vct{z}, y)\sim \text{source distribution}} l( \vct{\Pi}_i^\top \vct{z}, y) \\
    \nonumber
    = & \vct{\Pi}_i^\top \begin{bmatrix}
        1/2 & 0 & 0\\
        0 & 1/2 & 0 \\
        0 & 0 & \sigma^2
\end{bmatrix}\vct{\Pi}_i - [1, ~~1, ~~0] ~~\mtx{\Pi}_i + 1.
    % \nonumber
    % = &\E_{(\vct{z}, y)\sim \text{source distribution}} (\vct{\Pi}_i^\top \vct{z} -  \vct{\beta}_s^\top \vct{z})^2 \\
    % \nonumber
    % = &\E_{(\vct{z}, y)\sim \text{source distribution}} ((\vct{\Pi}_i-\vct{\beta}_s)^\top \vct{z})^2
\end{align}
Note that $\vct{\Pi}_i^*$'s (for $i \geq 2$) have to be orthogonal to $\vct{\Pi}_1^*$, which equals $[1, 1, 0]^\top$, while minimizing the above. Therefore, we substitute $[1, 1, 0] \mtx{\Pi}_i = 0$ into the above, which reveals that $\vct{\Pi}_i^*$ minimizes 
\begin{align}
    \nonumber
  \vct{\Pi}_i^*  =\arg\min_{\mtx{\Pi}_i} \bigg( \vct{\Pi}_i^\top \begin{bmatrix}
        1/2 & 0 & 0\\
        0 & 1/2 & 0 \\
        0 & 0 & \sigma^2
\end{bmatrix}\vct{\Pi}_i  + 1 \bigg).
\end{align}
Note that the RHS is minimized when $\vct{\Pi}_i = \vct{0}$. Therefore we conclude that $\vct{\Pi}_i^*=\vct{0}$ for $i\geq 2$. Given that, and  considering $\vct{v}_3^\top \vct{\Pi}_1 =0$  we obtain
\begin{align}
    \nonumber
\vct{v}_3^\top \mtx{\Pi}^{*}   = \vct{0}.
\end{align}
The linear model trained on top of the projected embeddings using the target data has the following expression
\begin{align}
\label{eq: pro2_w}
   \vct{w}^* = &  \hat{\E}_{\vct{z}^{t}\in\mathcal{E}^t}( \mtx{\Pi}^{*\top} \vct{z}^{t} \vct{z}^{t\top} \mtx{\Pi}^* ) ^\dagger \hat{\E}_{\vct{z}^{t}\in\mathcal{E}^t}(\mtx{\Pi}^{*\top} \vct{z}^{t} y^t  )
\end{align}
We first examine the projected covariance
\begin{align}
    \nonumber
\hat{\E}_{(\vct{z}^{t}, y)\in\mathcal{E}^t}( \mtx{\Pi}^{*\top} \vct{z}^{t} \vct{z}^{t\top} \mtx{\Pi}^* )
= \hat{\E}_{(\vct{z}^{t}, y)\in\mathcal{E}^t}
\begin{bmatrix}
    (\vct{z}^{t\top}\vct{\Pi}_1^* )^2 & 0 & 0\\
    0 & 0 & 0\\
    0 & 0 & 0
\end{bmatrix}
\end{align}

Noticing that $\vct{z}^{t\top}\vct{\Pi}_1^* = by^t +\xi$ where $b\sim \text{Bernoulli}(1/2)$, we can apply concentration bounds of random variables to deduce that $|\hat{\E}(\vct{z}^{t\top}\vct{\Pi}_1^*)^2 - (1/2+\sigma^2)| \leq o(1) $ with high probability, when $m$ is sufficiently large and $\sigma=o(1)$. Additionally, $\hat{\E}_{\vct{z}^{t}\in\mathcal{E}^t}(\mtx{\Pi}^{*\top} \vct{z}^{t} y^t  ) $ concentrates around $[1/2 ,0, \dots]^\top$ with error bounded by $o(1)$ with high probability, similar to what we show in the previous section. Finally, we can further calculate that $\vct{w}^*$ concentrates around $
\begin{bmatrix}
    \frac{1}{2\sigma^2+1}\\
    0\\
    \vdots
\end{bmatrix}
$ by Equation \ref{eq: pro2_w}. Thus, $\mtx{\Pi}^*\vct{w}^*$ concentrates around $\frac{1}{2\sigma^2+1} \vct{\Pi}_1^* =\frac{1}{2\sigma^2+1}
\begin{bmatrix}
    1\\
    1\\
    0
\end{bmatrix}
$.

Now we are ready to bound the test loss. Combining the above conclusions on concentration and the fact that $\sigma=o(1)$, we obtain
\begin{align}
    \nonumber
    \E_{(\vct{z, y}) \in \text{target dist.} } l (\vct{z}^\top \mtx{\Pi}^*\vct{w}^*, y) = & \frac{1}{2}\E_{\xi, y}([0, y, \xi]\mtx{\Pi}^*\vct{w}^* - y  )^2 + \frac{1}{2}\E_{\xi, y}([0, \xi, y]\mtx{\Pi}^*\vct{w}^* - y  )^2 \\
\nonumber
= & \frac{1}{2}(\frac{1}{2\sigma^2+1} - 1  )^2 + \frac{1}{2}\E_{\xi, y}( \frac{1}{2\sigma^2+1}\xi - y  )^2 \pm o(1)\\
\nonumber
= & \frac{1}{2}(\frac{1}{2\sigma^2+1} - 1  )^2 +\frac{1}{2}(\frac{\sigma^2}{(2\sigma^2+1)^2} +  1) \pm o(1)\\
\nonumber
= & 0.5 \pm o(1).
\end{align}

% \begin{align}
%         \begin{bmatrix}
%         \frac{1}{2}(1-s)^2+s^2 \frac{\sigma_\xi^2}{d-2} & \frac{1}{4}(1-s)s & \frac{1}{4} (1-s)s & 0 & \dots & 0 \\
%         \frac{1}{4} (1-s)s & \frac{1}{4}s^2 +\frac{1}{4}(1-s)^2 +\frac{1}{4} & \frac{1}{4}(1-s)s & 0 & \dots & 0 \\
%         \frac{1}{4} (1-s)s & \frac{1}{4} (1-s)s & \frac{1}{2}s^2+(1-s)^2\frac{\sigma_\xi^2}{d-2}  & 0 & \dots & 0\\
%          0 & 0 & 0  \\
%          \vdots & \vdots & \vdots \\
%          0 & 0 & 0 & 
%     \end{bmatrix}
% \end{align}

\subsection{Effect of $s$ when $\sigma = \omega(1)$} \label{apdx: explanation_large_sigma}

Here, to understand the benefit of choosing an intermediate $s$ instead of solely using the target data by setting $s=1$, we consider the case of large noise where $\sigma = \omega(1)$. The analysis is similar to that in Section \ref{apdx: domaingen_ours}, except that the dominant term in Equation \ref{eq: delta_sigma} is now $O(s^2\sigma^2\sqrt{\frac{\log 1/\delta}{m}})$. Additionally, in Equation \ref{eq: diff_inverse}, we compare $\hat{\mtx{\Sigma}}_\mixed^{-1} \hat{\vct{u}}_\mixed$ to $\mtx{\Sigma}_\mixed^{-1} \vct{u}_\mixed$ instead of $\mtx{\Sigma}_0^{-1}\vct{u}_\mixed$, i.e.,
\begin{align}
    \nonumber
    \|\hat{\mtx{\Sigma}}_\mixed^{-1} \hat{\vct{u}}_\mixed - \mtx{\Sigma}_\mixed^{-1}\vct{u}_\mixed\|_2 \leq & O( \| \Delta\mtx{\Sigma}_\mixed^{-1}  \|_2 ) + O( \| \Delta \vct{u}_\mixed \|_2 ) = O(s^2\sigma^2\sqrt{\frac{\log 1/\delta}{m}}).
\end{align}
Then, the final step becomes:
\begin{align}
    \nonumber
    \sqrt{\E_{(\vct{z, y}) \in \text{target dist.} } l (\vct{z}^\top \vct{w}_{\ours}^*, y)} \leq & \| \mtx{\Sigma_t}^{1/2}( \vct{\beta}_t - \mtx{\Sigma}_\mixed^{-1}\vct{u}_\mixed)  \| + \| \mtx{\Sigma_t}^{1/2}( \mtx{\Sigma}_\mixed^{-1}\vct{u}_\mixed - \hat{\mtx{\Sigma}}_\mixed^{-1}\hat{\vct{u}}_\mixed )   \|.  
\end{align}
Note that $\mtx{\Sigma}_\mixed^{-1}$ and therefore $\mtx{\Sigma}_\mixed^{-1}\vct{u}_\mixed$ can be calculated exactly and are independent of $m$, thus the above takes the form $\psi(s, \sigma) + O(s^2\sigma^2\sqrt{\frac{\log 1/\delta}{m}})$, where $\psi(s, \sigma)$ is a function of only $s$ and $\sigma$.

\subsection{Proof of Theorem \ref{thm: subpop}}\label{apdx: subpop}

\textbf{Expression of the test loss.} Consider the linear model $\vct{w}_{\ours}^*$ learned by \ours with $l()$ being MSE loss with $\ell_2$ regularization, i.e., we minimize $\min_{\vct{w}} \hat{\E}_{(\vct{z},y)\in\mathcal{E}_\mixed} (\vct{w}^\top \vct{z}-y)^2 + \lambda \|\vct{w}\|^2  $. With assumption \ref{assump: asymptotic}, the test loss on the target distribution achieved by  $\vct{w}_{\ours}^*$ has the following closed form expression:
\begin{align}
    \nonumber
    \E_{(\vct{z}, y)\in\mathcal{E}^t} (\vct{w}_{\ours}^{*\top} \vct{z}- y)^2 = &
\bigg( \frac{q}{\psi_1\psi_3-\psi_2^2}(\psi_3-\psi_2(1-s)(2p-1))  -1\bigg)^2 \\
\nonumber
+ & \bigg( \frac{q}{\psi_1\psi_3-\psi_2^2}(-\psi_2+\psi_1(1-s)(2p-1))   \bigg)^2 \\
\nonumber
+ & \sigma_1^2 \bigg( \frac{q}{\psi_1\psi_3-\psi_2^2}( \psi_3-\psi_2(1-s)(2p-1) ) \bigg)^2,
\end{align}
where 
\begin{align}
    \nonumber
    \psi_1 \coloneqq & (1-s)^2(1+\sigma_1^2) + 2s(1-s)q +s^2 q(1+\sigma_1^2) -(1-s)^2 (1-q)+\lambda \\
    \nonumber
    \psi_2 \coloneqq &  (1-s)^2(2p-1)+s(1-s)q(2p-1) - (1-s)^2 (1-q)(2p-1)\\
    \nonumber
    \psi_3 \coloneqq & (1-s)^2 + s^2q - (1-s)^2 (1-q)(2p-1)^2+\lambda,
\end{align}
with $q \coloneqq \frac{\lambda}{s^2 r+\lambda}$.

We prove the above conclusion below.

% We study ridge regression with parameter $\lambda$. 
% %To simplify the analysis, we consider the linear model without the bias term $b$, while noting that the final conclusion regarding the shape of the loss curve would remain the same even if we consider $b$. 
% We consider the asymptotic regime where $n, d, m \rightarrow \infty$, $n/d \rightarrow \infty$ and $d/m \rightarrow \infty$. We also assume $\sigma_\xi^2 \rightarrow \infty$ while $\frac{\sigma_\xi^2}{m}$ remains at constant $r$.

% \begin{theorem}[Theorem \ref{thm: loss_wg}]\label{thm: loss_wg_apdx}
% % \yx{if have time add a variant of the theorem where $(1,-1)$ is not present. }
% Worst-group test loss
% $\mathcal{L}^{test}_{wg}(\vct{w}^*)$ has the following closed form expression:
% \begin{align}
%     \nonumber
%     \mathcal{L}^{test}_{wg}(\vct{w}^*) \!=\!  \bigg(\! \frac{q}{\psi_1\psi_3\!-\!\psi_2^2}\big( \psi_3\!+\!\psi_2\!-\!(\psi_1\!+\!\psi_2)(1\!-\!s)(2p\!-\!1) \big)-1\!\bigg)^2 %\\
%     %&~~~~
%     \!\!\!\!+\!\sigma_1^2 \big(\psi_3-\psi_2(1\!-\!s)(2p\!-\!1)\big)^2\!,
% \end{align}
% where $q \coloneqq \frac{\lambda}{s^2 r+\lambda}$, and 
% \begin{align}
% % \nonumber
% % q \coloneqq & \frac{\lambda}{s^2 r+\lambda}\\
%     \nonumber
%     \psi_1 \coloneqq & (1-s)^2(1+\sigma_1^2) + 2s(1-s)q +s^2 q(1+\sigma_1^2) -(1-s)^2 (1-q)+\lambda, \\
%     \nonumber
%     \psi_2 \coloneqq &  (1-s)^2(2p-1)+s(1-s)q(2p-1) - (1-s)^2 (1-q)(2p-1),\\
%     \nonumber
%     \psi_3 \coloneqq & (1-s)^2 + s^2q - (1-s)^2 (1-q)(2p-1)^2+\lambda.
% \end{align}
% \end{theorem}

\textbf{Notations.} Let $\mtx{Z}\in\mathbbm{R}^{d\times n}$ collect inputs in $\mathcal{E}^s$. Let $\mtx{Z}' \in\mathbbm{R}^{d\times m}$ collect inputs in $\mathcal{E}^t$. Remember that for each example in $\mathcal{E}^s$, we sample an example from $\mathcal{E}^t$ to construct $\embmx$. Let $\mtx{Z}'' \in\mathbbm{R}^{d\times n}$ collect the $n$ sampled inputs (i.e., with duplicates) from $\mathcal{E}^t$. Let $\tilde{\mtx{Z}}$ denote the inputs in $\embmx$. Then $\tilde{\mtx{Z}} = (1-s)\mtx{Z} + s\mtx{Z}'' $. Let $\mtx{Y}, \mtx{Y}', \mtx{Y}''$ be the corresponding labels. Note that $\mtx{Y}'' = \mtx{Y}$. For convenience, in this section, we use $\vct{z},\vct{z}', \vct{z}'', \tilde{\vct{z}}$ for columns in $\mtx{Z}, \mtx{Z}', \mtx{Z}'', \tilde{\mtx{Z}}$, respectively (instead of the notations with $t$ and $s$ used in the main paper). 
We write $\vct{z}' = \cat{\vct{f}'}{ \vct{\zeta}' } $ where $\vct{f}'\in\mathbbm{R}^{2}$ and $\vct{\zeta}'\in\mathbbm{R}^{d-2}$. Let $\mtx{F}' = [\vct{f}'_1, \vct{f}'_2, \dots, \vct{f}'_m]$.  Let $\mathcal{C}_y$ denote the set of indices of examples in $\mathcal{E}^t$ with label $y$. Let $\mathcal{S}_j$ denote the set of indices of examples from $\mathcal{E}^s$ that are going to be mixed with $\vct{z}_j'$. 

\begin{proposition}\label{prop: asymp}
The following holds asymptotically almost surely (a.a.s.) in the asymptotic regime we consider
\begin{align}
    \label{eq: FF}
    & \frac{1}{m}\mtx{F}'\mtx{F}' = \frac{1}{m}\sum_{i=1}^m \vct{f}_i' \vct{f}_i'^{\top} =
    \begin{bmatrix}
    1+\sigma_1^2 & 0 \\
    0 & 1 \\
    \end{bmatrix} \coloneqq \mtx{G}_{core}  \\
    \label{eq: known_orthog}
    & \forall i\neq j \in[m], ~~ \vct{\zeta}_i'^{\top} \vct{\zeta}_j' = 0 \\
    \label{eq: known_norm}
    & \forall i \in[m], ~~ \|\vct{\zeta}_i'\| = \sigma_{\xi} \\
    \nonumber
   & \frac{1}{n}\mtx{Z}\mtx{Z}^\top = \frac{1}{n}\sum_{i=1}^n \vct{z}\vct{z}^\top = 
   \begin{bmatrix}
      \mtx{G}_{spu} & \mtx{0} \\
      \mtx{0} & \frac{\sigma_{\xi}^2}{d-2}\mtx{I}_{d-2}
    \end{bmatrix},
      \text{ where }~~ \mtx{G}_{spu} \coloneqq \begin{bmatrix}
          1+\sigma_1^2 & 2p-1 \\
          2p-1 & 1
   \end{bmatrix} \\
   \label{eq: sum_x_in_S}
   & \frac{m}{n} \sum_{i\in\mathcal{S}_j} \vct{z}_i = \begin{bmatrix}
       y_j' \\
       y_j'(2p-1)\\
       0 \\
       \vdots\\
       0
   \end{bmatrix} \coloneqq \bar{\vct{z}}_{y_j'} \\
   \label{eq: sum_gp}
   & \frac{2}{m} \sum_{j\in \mathcal{C}_y} \vct{f}_j' = 
   \begin{bmatrix}
       y \\
       0 
   \end{bmatrix} \\
   \label{eq: FY}
   & \frac{1}{m}\mtx{F}'\mtx{Y}' = \cat{ 1 }{ 0} \\
   \label{eq: XY}
   & \frac{1}{n}\mtx{Z}\mtx{Y} = \begin{bmatrix}
       1 \\
       2p-1 \\
       0 \\
       \vdots \\
       0
   \end{bmatrix}
\end{align} 
\end{proposition}

The following derivations are performed based on the equations in Proposition \ref{prop: asymp}. 
We index the examples such that the first $m/2$ examples in $\mathcal{E}^t$ have label $1$ and the rest have label $-1$. Since equations \ref{eq: known_orthog} and \ref{eq: known_norm} are true, without loss of generality, we assume $\vct{\zeta}'_i = \sigma_\xi \vct{e}_i$ where $\vct{e}_i$ is the $i$-th standard basis in $\mathbbm{R}^{d-2}$.
\begin{corollary}\label{cor: XY}
 The following is true
 \begin{align}
     \nonumber
     \frac{1}{n}\tilde{\mtx{Z}} \mtx{Y} = & 
     \begin{bmatrix}
         1 \\
        (1-s)(2p-1) \\
        s\frac{\sigma_\xi}{m}\mtx{Y}'\\
        \mtx{0}
     \end{bmatrix}.
 \end{align}
\end{corollary}
\begin{proof}
$\frac{1}{n}\mtx{Z}''\mtx{Y} = \frac{1}{m}\mtx{Z}'\mtx{Y}'$ holds almost surely in the limit we are considering. Then
\begin{align}
    \nonumber
    \frac{1}{n}\tilde{\mtx{Z}} \mtx{Y} = & (1-s)\frac{1}{n}\mtx{Z}\mtx{Y} + s\frac{1}{n}\mtx{Z}''\mtx{Y} \\
    \nonumber
    =  & (1-s)\frac{1}{n}\mtx{Z}\mtx{Z} + s\frac{1}{m}\mtx{Z}'\mtx{Y}' \\
    \nonumber
    = & (1-s) 
     \begin{bmatrix}
         1 \\
         2p-1 \\
         0 \\
         \vdots\\
         0
     \end{bmatrix} + s 
     \begin{bmatrix}
         1 \\
        0 \\
        \frac{\sigma_\xi}{m}\mtx{Y}'\\
        \mtx{0}
     \end{bmatrix} ~~~~ \text{by the assumption on $\vct{\zeta}_i'$'s and Equations \ref{eq: XY} and \ref{eq: FY}},
\end{align}
which completes the proof.
\end{proof}

\begin{lemma}
The following is true
\begin{align}
\nonumber
    \frac{1}{n}(\mtx{Z}\mtx{Z}''^\top + \mtx{Z}''\mtx{Z}^\top ) = 
    \begin{bmatrix}
        \mtx{H} & \mtx{N} & \mtx{0} \\
        \mtx{N}^\top & \mtx{0} & \mtx{0} \\
        \mtx{0} & \mtx{0} & \mtx{0} \\
    \end{bmatrix},
\end{align}
where 
\begin{align}
    \mtx{H} & \coloneqq 
    \begin{bmatrix}
        2 & 2p-1 \\
        2p-1 & 0
    \end{bmatrix} \\
\label{eq: def_N}
    \mtx{N} & \coloneqq \frac{\sigma_\xi}{m} 
    \begin{bmatrix}
        1 \\
        2p-1 
    \end{bmatrix}
    \big[
       \underbrace{1 ~~ 1 \dots ~1}_{\text{ the first $m/2$ elements are $1$}}  \underbrace{-1 ~~ -1 \dots ~ -1}_{\text{the last $m/2$ elements are $-1$}}   
    \big] =   \frac{\sigma_\xi}{m} 
    \begin{bmatrix}
        1 \\
        2p-1 
    \end{bmatrix} \mtx{Y}'^\top
\end{align}
\end{lemma}
\begin{proof} We first derive the expression of $\frac{1}{n}\mtx{Z}\mtx{Z}''^\top$
\begin{align}
\nonumber
    \frac{1}{n}\mtx{Z}\mtx{Z}''^\top = & \frac{1}{2}\sum_{y\in\{1,-1\}} \frac{2}{m} \sum_{j\in \mathcal{C}_y}  (\frac{m}{n} \sum_{i\in\mathcal{S}_j} \vct{z}_i ) \vct{z}_j'^\top \\
    \nonumber
    = & \frac{1}{2} \sum_{y\in\{1,-1\}} \bar{\vct{z}}_{y} \frac{2}{m} \sum_{j\in\mathcal{C}_y} \vct{z}_j'^\top  ~~~~~~~ \text{by equation \ref{eq: sum_x_in_S}} \\
    \nonumber
    = & \frac{1}{2} \sum_{y\in\{1,-1\}} \bar{\vct{z}}_{y} \frac{2}{m} \sum_{j\in\mathcal{C}_y} [\vct{f}_j'^\top ~~\vct{\zeta}_j'^\top]  \\
    \nonumber
    = & \frac{1}{2} \sum_{y\in\{1,-1\}} \bar{\vct{z}}_{y} [ y ~~ 0  ~~~ \frac{2}{m}\sum_{j\in\mathcal{C}_y} \sigma_{\xi} \vct{e}_{j}^\top ] ~~~~~~~ \text{by equation \ref{eq: sum_gp} and the assumption about $\vct{\zeta}_j'$ } \\
    \nonumber
    = & 
    \begin{bmatrix}
        1 & 0 & \frac{\sigma_\xi}{2} & \frac{\sigma_\xi}{2} & \dots & -\frac{\sigma_\xi}{2} & -\frac{\sigma_\xi}{2} & \dots & 0 & \dots \\
        2p-1 & 0 & (2p-1)\frac{\sigma_\xi}{2} & (2p-1)\frac{\sigma_\xi}{2} & \dots & -(2p-1)\frac{\sigma_\xi}{2} & -(2p-1)\frac{\sigma_\xi}{2} & \dots & 0 & \dots \\
        0 & \dots &  & & & & & & 0 & \dots \\ 
        \vdots & &  & & & & & & \vdots & \ddots \\
        0 & \dots &  & & & & & & 0 & \dots \\
    \end{bmatrix}.
\end{align}
The expression of $\frac{1}{n}\mtx{Z}''\mtx{Z}^\top $ is just the transpose of the above. Adding the two expressions together completes the proof.
\end{proof}

\begin{lemma}
The following is true
\begin{align}
    \nonumber
    \frac{1}{n}\mtx{Z}''\mtx{Z}''^\top = \begin{bmatrix}
        \mtx{G}_{core} &  \frac{\sigma_\xi}{m} \mtx{F}' & \mtx{0} \\
         \frac{\sigma_\xi}{m} \mtx{F}'^\top & \frac{\sigma_\xi^2}{m} \mtx{I}_m & \mtx{0} \\
         \mtx{0} & \mtx{0} & \mtx{0}
     \end{bmatrix}.
\end{align}
\end{lemma}
\begin{proof}
$\frac{1}{n}\mtx{Z}''\mtx{Z}''^\top =  \frac{1}{m}\mtx{Z}'\mtx{Z}'^\top$ almost surely in the limit we are considering. Then
\begin{align}
    \nonumber
     \frac{1}{n}\mtx{Z}''\mtx{Z}''^\top = & \frac{1}{m}\mtx{Z}'\mtx{Z}'^\top \\
     \nonumber
     = & \frac{1}{m}
     \begin{bmatrix}
         \mtx{F}' \\
         \sigma_\xi \mtx{I}_m \\
         \mtx{0}
     \end{bmatrix} [\mtx{F}'^\top ~~~
         \sigma_\xi \mtx{I}_m ~~~
         \mtx{0}^\top] \\
     \nonumber
     = & \frac{1}{m}
     \begin{bmatrix}
         \mtx{F}' \mtx{F}'^\top & \sigma_\xi \mtx{F}' & \mtx{0} \\
         \sigma_\xi \mtx{F}'^\top & \sigma_\xi^2 \mtx{I}_m & \mtx{0} \\
         \mtx{0} & \mtx{0} & \mtx{0}
     \end{bmatrix} \\
     \nonumber
     = & 
     \begin{bmatrix}
        \mtx{G}_{core} &  \frac{\sigma_\xi}{m} \mtx{F}' & \mtx{0} \\
         \frac{\sigma_\xi}{m} \mtx{F}'^\top & \frac{\sigma_\xi^2}{m} \mtx{I}_m & \mtx{0} \\
         \mtx{0} & \mtx{0} & \mtx{0}
     \end{bmatrix}
\end{align}
\end{proof}

Define $\mtx{M} \coloneqq \frac{1}{n}\tilde{\mtx{Z}}\tilde{\mtx{Z}}^\top + \lambda \mtx{I}_d $. Then
\begin{align}
\nonumber
\mtx{M}= &\frac{(1-s)^2}{n}\mtx{Z}\mtx{Z}^\top + \frac{s(1-s)}{n}\mtx{Z}\mtx{Z}''^\top + \frac{s(1-s)}{n}\mtx{Z}''\mtx{Z}^\top + \frac{s^2}{n}\mtx{Z}''\mtx{Z}''^\top +\lambda \mtx{I}_d \\
    \nonumber
    = & (1-s)^2 \begin{bmatrix}
      \mtx{G}_{spu} & \mtx{0} \\
      \mtx{0} & \frac{\sigma_{\xi}^2}{d-2}\mtx{I}_{d-2}
    \end{bmatrix}
    + s(1-s) \begin{bmatrix}
        \mtx{H} & \mtx{N} & \mtx{0} \\
        \mtx{N}^\top & \mtx{0} & \mtx{0} \\
        \mtx{0} & \mtx{0} & \mtx{0} \\
    \end{bmatrix} + s^2 \begin{bmatrix}
        \mtx{G}_{core} &  \frac{\sigma_\xi}{m} \mtx{F}' & \mtx{0} \\
         \frac{\sigma_\xi}{m} \mtx{F}'^\top & \frac{\sigma_\xi^2}{m} \mtx{I}_m & \mtx{0} \\
         \mtx{0} & \mtx{0} & \mtx{0}
     \end{bmatrix} +\lambda \mtx{I}_d \\
     \nonumber
     = & \begin{bmatrix}
        \mtx{A} & \mtx{B} & \mtx{0} \\
        \mtx{B}^\top & \mtx{D} &\mtx{0}\\
        \mtx{0} & \mtx{0} & \frac{\sigma_\xi^2}{d-2}\mtx{I}_{d-2-m}+\lambda \mtx{I}_{d-2-m}
     \end{bmatrix},
\end{align}
where 
\begin{align}
    \nonumber
    \mtx{A} \coloneqq & (1-s)^2 \mtx{G}_{spu} + s(1-s) \mtx{H} + s^2 \mtx{G}_{core} + \lambda \mtx{I}_2 \\
    \nonumber
    \mtx{B} \coloneqq & s(1-s) \mtx{N} + s^2 \frac{\sigma_\xi}{m}\mtx{F}'\\
    \nonumber
    \mtx{D} \coloneqq & \big((1-s)^2\frac{\sigma_\xi^2}{d-2} + s^2 \frac{\sigma_\xi^2}{m} + \lambda \big)\mtx{I}_m = (s^2 r +\lambda) \mtx{I}_m
\end{align}

The inverse of $\mtx{M} $ is
\begin{align}
\label{eq: inv_M}
    \mtx{M}^{-1} = 
    \begin{bmatrix}
        (\mtx{P}/\mtx{D} )^{-1} &  -(\mtx{P}/\mtx{D} )^{-1}\mtx{B}\mtx{D}^{-1} & \mtx{0}\\
        -\mtx{D}^{-1}\mtx{B}^\top (\mtx{P}/\mtx{D} )^{-1} & \mtx{D}^{-1} + \mtx{D}^{-1}\mtx{B}^\top (\mtx{P}/\mtx{D} )^{-1} \mtx{B}\mtx{D}^{-1} & \mtx{0}\\
        \mtx{0} & \mtx{0} & \frac{1}{\frac{\sigma_\xi^2}{d-2}+\lambda} \mtx{I}_{d-2-m}
    \end{bmatrix},
\end{align}
where $\mtx{P}\coloneqq 
\begin{bmatrix}
    \mtx{A} & \mtx{B} \\
    \mtx{B}^\top & \mtx{D}
\end{bmatrix}
$ and $\mtx{P}/\mtx{D} \coloneqq \mtx{A}-\mtx{B}\mtx{D}^{-1} \mtx{B}^\top$ is the Schur complement of $\mtx{D}$ in $\mtx{P}$. 

We first derive the expressions of $\mtx{F}'\mtx{N}^\top$.
\begin{align}
    \nonumber
    \mtx{F}'\mtx{N}^\top = & \frac{\sigma_\xi}{m} \mtx{F}' \mtx{Y}' [1 ~~ 2p-1]  \\
    \label{eq: FN}
    = & \sigma_\xi 
    \begin{bmatrix}
        1 & 2p-1 \\
        0 & 0 
    \end{bmatrix} ~~~~\text{by Equation \ref{eq: FY}}
\end{align}

Now we derive $\mtx{B}\mtx{B}^\top$.
\begin{align}
    \nonumber
    \mtx{B}\mtx{B}^\top = & s^2(1-s)^2\mtx{N}\mtx{N}^\top + s^4 \frac{\sigma_\xi^2}{m^2}\mtx{F}'\mtx{F}'^\top + s^3(1-s)\frac{\sigma_\xi}{m}(\mtx{N}\mtx{F}'^\top + \mtx{F}'\mtx{N}^\top) \\
    = & s^2(1-s)^2 r \mtx{J} + s^4 r \mtx{G}_{core} + s^3(1-s) r \mtx{H} ~~~~\text{by Equations \ref{eq: def_N}, \ref{eq: FF} and \ref{eq: FN} },
\end{align}
where $\mtx{J} \coloneqq 
\begin{bmatrix}
    1 & 2p-1 \\
    2p-1 & (2p-1)^2
\end{bmatrix}
$. 

Let $q \coloneqq \frac{\lambda}{s^2 r+\lambda}$. Then
\begin{align}
\nonumber
    \mtx{P}/\mtx{D} = & (1-s)^2 \mtx{G}_{spu} + s(1-s)\frac{\lambda}{s^2 r +\lambda} \mtx{H} + s^2 \frac{\lambda}{s^2 r +\lambda} \mtx{G}_{core} - (1-s)^2 \frac{s^2 r}{s^2 r+\lambda} \mtx{J} + \lambda \mtx{I}_2 \\
    \nonumber
    = & (1-s)^2 \mtx{G}_{spu} + s(1-s)q\mtx{H} + s^2 q \mtx{G}_{core} - (1-s)^2 (1-q) \mtx{J} + \lambda \mtx{I}_2 \\
    \nonumber
    = & 
    \begin{bmatrix}
        \psi_1 & \psi_2 \\
        \psi_2 & \psi_3
    \end{bmatrix},
\end{align}
where 
\begin{align}
    \nonumber
    \psi_1 \coloneqq & (1-s)^2(1+\sigma_1^2) + 2s(1-s)q +s^2 q(1+\sigma_1^2) -(1-s)^2 (1-q)+\lambda \\
    \nonumber
    \psi_2 \coloneqq &  (1-s)^2(2p-1)+s(1-s)q(2p-1) - (1-s)^2 (1-q)(2p-1)\\
    \nonumber
    \psi_3 \coloneqq & (1-s)^2 + s^2q - (1-s)^2 (1-q)(2p-1)^2+\lambda.
\end{align}
Then
\begin{align}
    \nonumber
    (\mtx{P}/\mtx{D})^{-1} = 
    \frac{1}{\psi_1\psi_3 - \psi_2^2}
    \begin{bmatrix}
        \psi_3 & -\psi_2 \\
        -\psi_2 & \psi_1
    \end{bmatrix}.
\end{align}

Let $\vct{w}$ be the minimizer of the objective in \ref{eq: train_on_mixed}, which has the closed form expression $\mtx{M}^{-1} \frac{1}{n}\tilde{\mtx{Z}}\mtx{Y}$. We are ready to derive the elements in $\vct{w}$ using Corollary \ref{cor: XY} and \ref{eq: inv_M}. But before that, we first derive the following
\begin{align}
    \nonumber
    \mtx{B}\mtx{Y}' = & s(1-s)\mtx{N}\mtx{Y}' + s^2 \frac{\sigma_\xi}{m} \mtx{F}'\mtx{Y}'  \\
    \nonumber
    = & s(1-s) \sigma_\xi 
    \begin{bmatrix}
        1 \\
        2p-1 
    \end{bmatrix} + s^2\sigma_\xi
    \begin{bmatrix}
        1\\
        0
    \end{bmatrix} \\
    \nonumber
    = & s\sigma_\xi
    \begin{bmatrix}
        1 \\
        (1-s)(2p-1)
    \end{bmatrix}.
\end{align}
Now we can drive the first two elements in $\vct{w}$ with the above and Corollary \ref{cor: XY} and \ref{eq: inv_M}
\begin{align}
\label{eq: w12}
    \vct{w}_{1:2} = & (\mtx{P}/\mtx{D})^{-1} 
    \begin{bmatrix}
        1\\
        (1-s)(2p-1)
    \end{bmatrix}   -(\mtx{P}/\mtx{D} )^{-1}\mtx{B}\mtx{D}^{-1} s \frac{\sigma_\xi}{m}\mtx{Y}' \\
    \nonumber
    = & q (\mtx{P}/\mtx{D})^{-1} \begin{bmatrix}
        1\\
        (1-s)(2p-1)
    \end{bmatrix}  \\
    \nonumber
    = & \frac{q}{\psi_1\psi_3-\psi_2^2}
    \begin{bmatrix}
        \psi_3-\psi_2(1-s)(2p-1)\\
        -\psi_2 + \psi_1(1-s)(2p-1)
    \end{bmatrix},
\end{align}
as well as the next $m$ entries
\begin{align}
    \nonumber
    \vct{w}_{3:m+2} = & -\mtx{D}^{-1}\mtx{B}^\top (\mtx{P}/\mtx{D} )^{-1} 
    \begin{bmatrix}
        1\\
        (1-s)(2p-1)
    \end{bmatrix}
    + \mtx{D}^{-1} s\frac{\sigma_\xi}{m}\mtx{Y}'+ \mtx{D}^{-1}\mtx{B}^\top (\mtx{P}/\mtx{D} )^{-1} \mtx{B}\mtx{D}^{-1} s\frac{\sigma_\xi}{m}\mtx{Y}' \\
    \nonumber
    = & \frac{\sigma_\xi}{m}\begin{bmatrix}
        h_1\\
        h_2\\
        \vdots\\
        h_m
    \end{bmatrix},
\end{align}
where
\begin{align}
\nonumber
c_1 \coloneqq & \psi_3 - \psi_2(1-s)(2p-1) \\
\nonumber
c_2 \coloneqq & -\psi_2+\psi_1(1-s)(2p-1) \\
    \nonumber
    h_i \coloneqq & \big(  \frac{s}{s^2 r+\lambda} - \frac{q}{s^2+\lambda}\frac{1}{\psi_1\psi_3-\psi_2^2}s(1-s)(c_1+(2p-1)c_2)  \big) y_i - s^2(f_{i,1}c_1+f_{i,2}c_2).
\end{align}
Note that the remaining entries in $\vct{w}$ are all zero.

% With simple calculation, we can obtain the MSE test loss on minority groups 
% \begin{align}
%     \nonumber
%     \mathcal{L}^{test}_{mnr} = & (w_1-w_2-1)^2 + \sigma_1^2 w_1^2 +  \frac{\sigma_\xi^2}{d-2}\sum_{i=3}^{m+2} w_i^2.
% \end{align}
% Let's look at the last term, which can be written as
% \begin{align}
%     \nonumber
%      \frac{\sigma_\xi^2}{d-2}\sum_{i=3}^{m+2} w_i^2 = & \frac{\sigma_\xi^2}{d-2}r \frac{1}{m}\sum_{i=i}^{m}h_i^2.
% \end{align}
% Since $\mathbbm{E}h$ is a constant and by the law of large number $\frac{1}{m}\sum_{i=i}^{m}h_i^2$ converges almost surely to the expected value, $\frac{1}{m}\sum_{i=i}^{m}h_i^2$ is also a constant. Note that $r$ is a constant, too. Then the RHS converges to $0$ because $\frac{\sigma_\xi^2}{d-2}= \frac{rm}{d-2}\rightarrow 0$ by our assumptions. Combining the above and equations \ref{eq: w12}, we obtain the following expression of $\mathcal{L}^{test}_{mnr}$ 
% \begin{align}
%      \mathcal{L}^{test}_{mnr}
%      =  \bigg( \frac{q}{\psi_1\psi_3-\psi_2^2}\big( \psi_3+\psi_2-(\psi_1+\psi_2)(1-s)(2p-1) \big)-1\bigg)^2 + \sigma_1^2 \big(\psi_3-\psi_2(1-s)(2p-1)\big)^2,
% \end{align}
% which completes the proof.

With simple calculation, we can obtain the MSE test loss on the target distribution
\begin{align}
    \nonumber
    \E_{(\vct{z}, y)\in\mathcal{E}^t} (\vct{w}^\top \vct{z}- y)^2 = (w_1-1)^2 + w_2^2 +\sigma_1^2w_1^2 + \frac{\sigma_\xi^2}{d-2}\sum_{i=3}^{m+2} w_i^2.
\end{align}
Let's look at the last term, which can be written as
\begin{align}
    \nonumber
     \frac{\sigma_\xi^2}{d-2}\sum_{i=3}^{m+2} w_i^2 = & \frac{\sigma_\xi^2}{d-2}r \frac{1}{m}\sum_{i=i}^{m}h_i^2.
\end{align}
Since $\mathbbm{E}h$ is a constant and by the law of large number $\frac{1}{m}\sum_{i=i}^{m}h_i^2$ converges almost surely to the expected value, $\frac{1}{m}\sum_{i=i}^{m}h_i^2$ is also a constant. Note that $r$ is a constant, too. Then the RHS converges to $0$ because $\frac{\sigma_\xi^2}{d-2}= \frac{rm}{d-2}\rightarrow 0$ by our assumptions. Combining the above and equations \ref{eq: w12}, we obtain the following expression of
\begin{align}
    \nonumber
     \E_{(\vct{z}, y)\in\mathcal{E}^t} (\vct{w}^\top \vct{z}- y)^2 = &
\bigg( \frac{q}{\psi_1\psi_3-\psi_2^2}(\psi_3-\psi_2(1-s)(2p-1))  -1\bigg)^2  + \bigg( \frac{q}{\psi_1\psi_3-\psi_2^2}(-\psi_2+\psi_1(1-s)(2p-1))   \bigg)^2 \\
\nonumber
+ & \sigma_1^2 \bigg( \frac{q}{\psi_1\psi_3-\psi_2^2}( \psi_3-\psi_2(1-s)(2p-1) ) \bigg)^2.
\end{align}
which completes the proof.

\section{Experimental Details}\label{apdx: exp_details}

\subsection{Datasets}\label{apdx: details_data}

\textbf{Waterbirds.} In the WaterBirds dataset \cite{sagawa2019distributionally}, the task is to classify images of birds as either landbirds or waterbirds, where each bird image has a background of either $land$ or $water$. The dataset introduces a challenge as the label is spuriously correlated with the image background. Among the 4,795 training samples, about 95\% exhibit this spurious correlation, with the numbers of the four combinations of background and label—(land background, landbird), (water background, landbird), (land background, waterbird), and (water background, waterbird)—being 3498, 184, 56, and 1057, respectively. 

\textbf{UrbanCars.} This dataset, constructed by \cite{li2023whac}, features multiple spurious correlations. The task is classifying images as either urban cars or country cars. Each image has one background (BG) and one co-occurring object (CoObj). The BG is selected from either urban or country backgrounds, and the CoObj is selected from either urban or country objects.  For each class, images with common BG and CoObj constitute 90.25\%, images with uncommon BG and common CoObj, or with common BG and uncommon CoObj, constitute 4.75\%, and images with both uncommon BG and CoObj constitute 0.25\%. This dataset presents a challenge due to multiple spurious correlations/shortcuts.

\textbf{bFFHQ.} This dataset, constructed by \cite{kim2021biaswap}, is the gender-biased version of the Flickr-Faces-HQ (FFHQ) dataset \cite{karras2019style}. The task is to classify whether a given facial image is `young' (aged 10–29) or `old' (aged 40–59). The label `young' is highly correlated with the attribute `women', and `old' is highly correlated with the attribute `man'.

\textbf{Camelyon17.} This dataset, included in the WILDS benchmark by \cite{koh2021wilds}, comprises medical images collected from various hospitals, leading to naturally occurring distribution shifts due to the differences in data collection processes. In the training set, the images are patches taken from 30 WSIs (whole-slide images), with 10 WSIs from each of the 3 hospitals in the
training set. In its out-of-distribution test set, the images are patches taken from 10 WSIs from a different hospital, which was chosen because its patches were the most visually distinctive. These WSIs are also distinct from those in the training set.

\textbf{PACS.} This dataset was constructed by \cite{li2017deeper} and is included in the DOMAINBED benchmark by \cite{gulrajani2020search}. The task is image classification across seven classes. It features images from four domains: Art (A), Cartoons (C), Photos (P), and Sketches (S).

\textbf{VLCS.} This dataset was constructed by \cite{fang2013unbiased} and is included in the DOMAINBED benchmark by \cite{gulrajani2020search}. The task is image classification across five classes. It features images from four photographic domains: Caltech101 (C), LabelMe (L), SUN09 (S), and
VOC2007 (V).

\textbf{Terra Incognita.} This dataset was constructed by \cite{beery2018recognition} and is included in the DOMAINBED benchmark by \cite{gulrajani2020search}.  The task is image classification across ten classes. It contains photographs of wild animals taken by camera traps at four different locations: L100, L38, L43, and L46.

\textbf{Office-Home.}  This dataset was constructed by \cite{venkateswara2017deep} and is included in the DOMAINBED benchmark by \cite{gulrajani2020search}. The task is image classification across 65 classes. It features images from four domains: Art (A), Clipart (C), Product (P) and Real (R).

\subsection{Details of data distribution}\label{apdx: details_dist}

WaterBirds, bFFHQ and UrbanCars  exhibit subpopulation shifts. In these datasets, certain subpopulations are significantly underrepresented in the source distribution. Following \cite{chen2023project}, we examine two types of target distribution that differ from the source distribution: (1) Minority (M), which only contains the minority subpopulations, and (2) Balanced (B), which contains an equal number of examples from each subpopulation. 

PACS, VLCS, Terra Incognita, and Office-Home are domain generalization datasets, each consisting of data collected from multiple distinct environments/domains. Following prior work \cite{gulrajani2020search,li2017deeper,fang2013unbiased,beery2018recognition,venkateswara2017deep}, we select one domain as the target distribution, with the remaining domains forming the source distribution. Below are the exact setups, following the common settings used in the aforementioned works. For PACS, we let P, A, and C constitute the source domain, with S as the target domain; for VLCS, we let V, L, and C constitute the source domain, with S as the target domain; for Terra Incognita, we let L100, L38, and L43 constitute the source domain, with L46 as the target domain; for Office-Home, we let A, C, and P constitute the source domain, with R as the target domain.

Camelyon17 is also a domain generalization dataset, where the source and target distributions consist of data collected from different hospitals.

\subsection{Training details}\label{apdx:training_details}

\textbf{Finetuning of ResNet50 on source data.} On WaterBirds, bFFHQ, and UrbanCars, we train the model on the source data using the SGD optimizer with a batch size of 128, a learning rate of 0.001 and weight decay of 0.0001 for 100 epochs. On Camelyon17, we use the same setting except we train for 20 epochs. On PACS, VLCS, Office-Home, and Terra Incognita, we use the same setting but train for 50 epochs.

\textbf{Training details of all methods.} For all methods, following \cite{chen2023project}, we employ the Adam optimizer \cite{kingma2014adam} with a batch size of 64 and train for 100 epochs. On Camelyon17, due to the large size of the source data, we randomly subsample 20,000 examples from the source data to mix with the target data when applying our method, \ours. This turns out to be sufficient to achieve very high performance.

\subsection{Details of hyperparameter tuning}\label{apdx: hp}

Table \ref{tab: hp} displays the hyperparameter ranges we use for tuning each method. For every method, we tune the learning rate and weight decay, along with method-specific hyperparameters.  It's important to note that the ranges we selected here are all based on the original paper, covering the ranges or the exact values considered in the original study. For \cite{teney2022evading}, in addition to the regularization strength $\lambda$, there is another hyperparameter, which is the number of diverse models trained on top of the embeddings. We set this number to 96, as \cite{teney2022evading} demonstrates that a higher number of models leads to better performance, and 96 is the maximum value they used.

\subsection{Details of the toy experiments in Figure \ref{fig:synthetic_pro2}}\label{apdx: synthetic}

We generate synthetic data that is a higher-dimensional generalization of that analyzed in Section \ref{sec: theory_domain_gen}, to demonstrate the generality of our results. Formally, each example $\vct{z}\in\mathbbm{R}^d$ is given by:
\begin{align}
    \nonumber
    \vct{z} = y\vct{v} + \vct{\xi},
\end{align}
where $\vct{v}$ is selected from $\{\vct{v}_1, \vct{v}_2\}$ in the source distribution and from $\{\vct{v}_2, \vct{v}_3\}$ in the target distribution. The noise $\vct{\xi}$ is a random Gaussian vector sampled from the space orthogonal to the features, formally $\vct{\xi}\sim\mathcal{N}(0, \frac{\sigma^2}{d-2}(\mtx{I}-\mtx{V}\mtx{V}^\top))$, where $\mtx{V}=[\vct{v}_1~~\vct{v}_2]$ in the source domain and $\mtx{V}=[\vct{v}_2~~\vct{v}_3]$ in the target domain.

In our experiments, we set $\sigma=4$ ,$d=8$ and $n=4000$.

\begin{table}
\small
    \caption{Hyperparameter range for each method. m.s. represents `method-specific'. }
    \label{tab: hp}
    \centering
    \begin{tabular}{|c|c|c|c|c|c|}
    \hline
       & \prot\ & DFR & Mixup & \citet{teney2022evading} & \ours \\
       \hline
       \hline
        lr &  \multicolumn{5}{c|}{$\{0.1, 0.01, 0.001\}$} \\
        \hline
        wd &  \multicolumn{5}{c|}{$\{0.1, 0.01, 0.001\}$}  \\
        \hline
        m.s. & $d\!\in\!\{1,2^2,\!2^4,\!2^6,\!2^8,\!2^{10}\}$ & None & $\alpha \in \{0.2,\! 0.4,\! 2^2,\! 2^3,\! 2^5\}$ & $\lambda\!\in\!\{5e^{-3},\!1e^{-2},\!0.1,\!1,\!5\}$ &  $s\!\in\!\{0.1, 0.3, 0.5, 0.7,0.9\}$\\
        \hline
    \end{tabular}
\end{table}

\section{Additional Experimental Results}\label{apdx: addtional_exp}

\begin{figure}[!t]
    \centering
    \includegraphics[width=.7\textwidth]{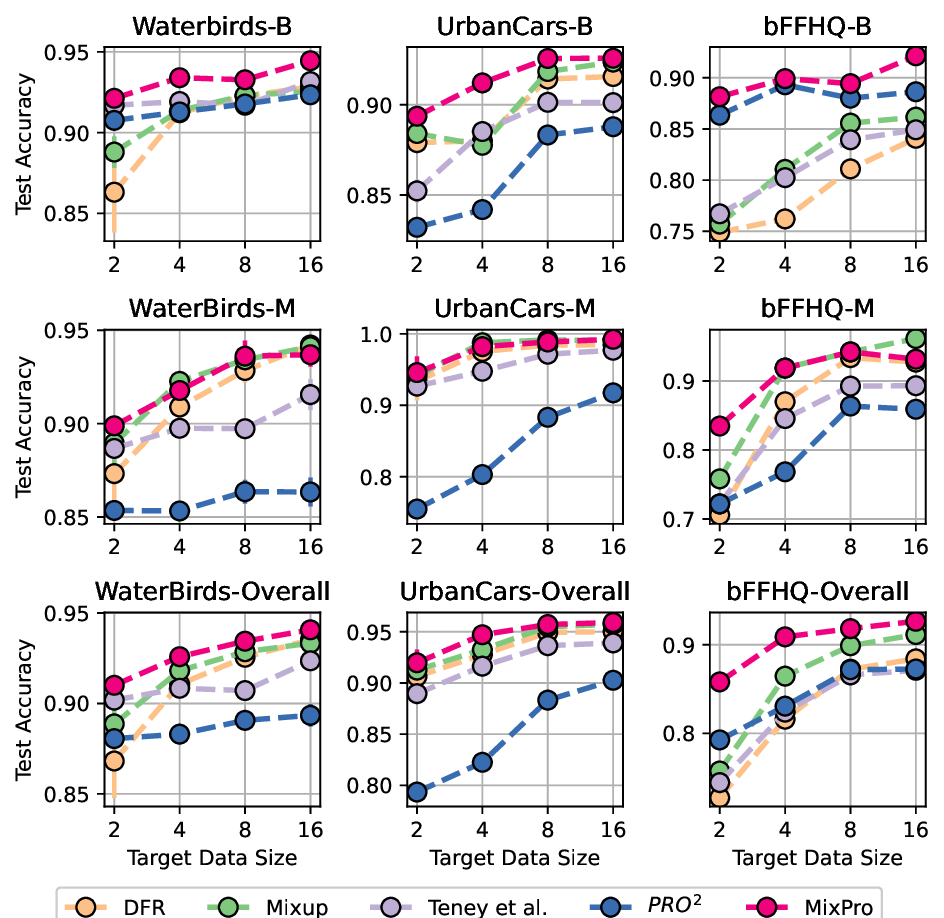}
    \caption{Detailed results on the three subpopulation shift datasets where we use the SWAG-pretrained VIT-L/16 as the backbone. The top section shows the results when the target distribution is balanced. The middle section displays the results when the target distribution contains only subpopulations that are minorities in the source data. The bottom section shows the average results over these two cases. We note that these bottom plots are the same as the corresponding plots in Figure \ref{fig:large_val_swagvit}. }
    \label{fig:subpop_all}
\end{figure}

\subsection{Detailed results on subpopulation shift datasets}\label{apdx: exp_subpop_all}

As mentioned in Section \ref{apdx: details_dist}, for each subpopulation shift dataset, we examined two types of target distributions: Balanced (B) and Minority (M). The complete results are presented in Figure \ref{fig:subpop_all}. We observe that \ours\ consistently performs the best on each dataset, while other methods may perform well in some cases but fail in others. Overall, as indicated by the bottom row of plots, \ours\ demonstrates the best performance.

\subsection{Results for ResNet50}\label{apdx: rn50}

We also present results for cases where we use the ImageNet-pretrained ResNet50, fine-tuned on the source data, as the backbone. Fine-tuning on the source data aligns with the practices in \cite{rosenfeld2022domain,kirichenko2022last} and is essential for achieving good performance; otherwise, the performance would be significantly lower, unlike with the SWAG-pretrained ViT. 

Figure \ref{fig:large_val_resnet} presents the results obtained by using an additional validation set for fine-tuning, while Figure \ref{fig:cross_val_resnet} shows the results from using cross-validation with the few given target data. We observe that, similar to the conclusions drawn in Figures \ref{fig:large_val_swagvit} and \ref{fig:cross_val_swagvit}, some methods may perform comparably to \ours\ on certain datasets but are surpassed on others.

\begin{figure}[!t]
    \centering
    \includegraphics[width=.98\textwidth]{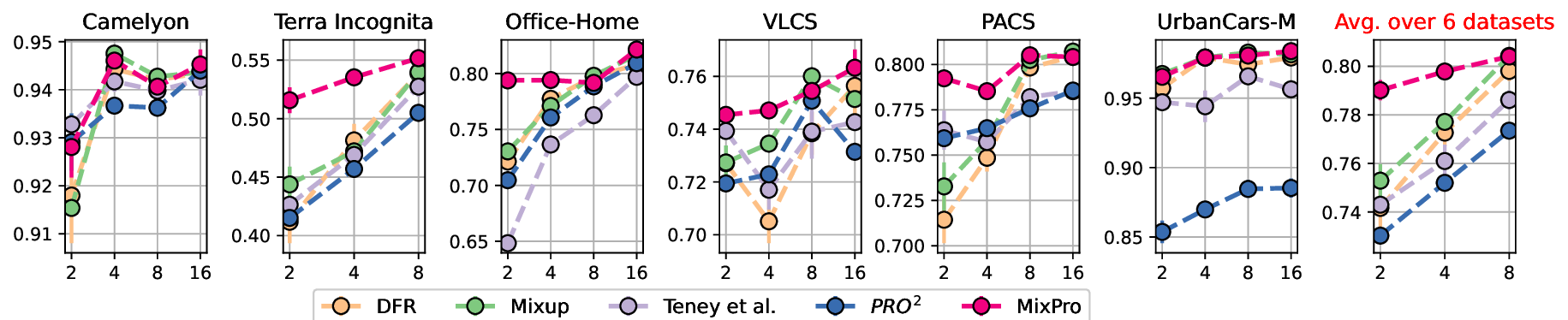}
    \caption{ Results are presented for cases where we use ResNet50, pretrained on ImageNet and fine-tuned on the source data, as the backbone model. Here we use the additional validation set for hyperparameter tuning.  We evaluate all methods across 6 datasets. We see that  \ours\ is the only one that consistently delivers the best performance. }
    \label{fig:large_val_resnet}
\end{figure}

\begin{figure}[!t]
    \centering
    \includegraphics[width=.99\textwidth]{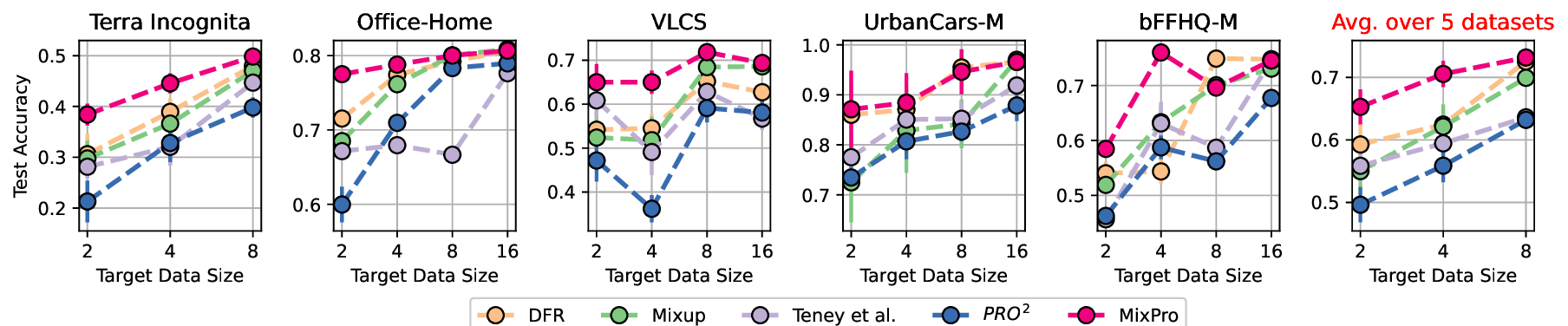}
    \caption{ Results are presented for cases where we use ResNet50, pretrained on ImageNet and fine-tuned on the source data, as the backbone model. Here we use cross-validation with the few given target data for hyperparameter tuning.}
    \label{fig:cross_val_resnet}
\end{figure}

\subsection{Mixing weight and target dataset size}\label{apdx: ablation}

Here we domonstrate how the size of the target data influences the choice of \( s \).  As depicted in Fig \ref{subfig: vary_r}, our theoretical analysis in Section \ref{sec: theory_subpop} has already indicated that with fewer target data points, a smaller \( s \) is preferable.  This means placing more emphasis on the source data to counteract the noise.  We confirm this intuition on real datasets. In Figure \ref{fig: s_size}, we show the test accuracy under different values of $s$ and target data sizes, for the results with ResNet50 as the backbone on three datasets.

\begin{figure}[!t]
    \centering
\subfigure[Terra Incognita]{    \includegraphics[width=.22\textwidth]{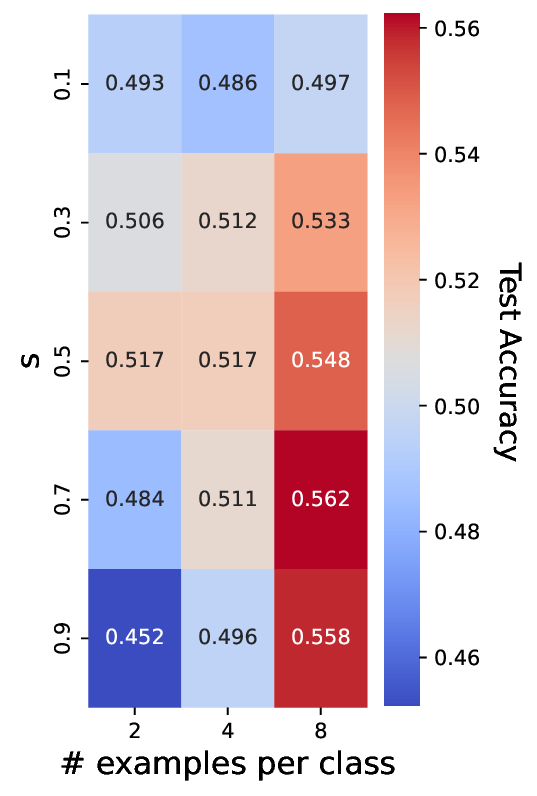}
}  
\subfigure[VLCS]{ \includegraphics[width=.22\textwidth]{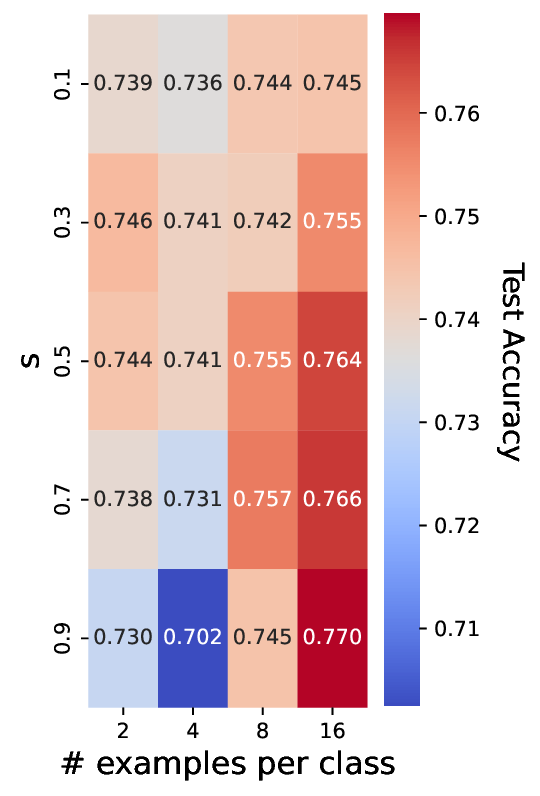}
}
\subfigure[PACS]{ \includegraphics[width=.22\textwidth]{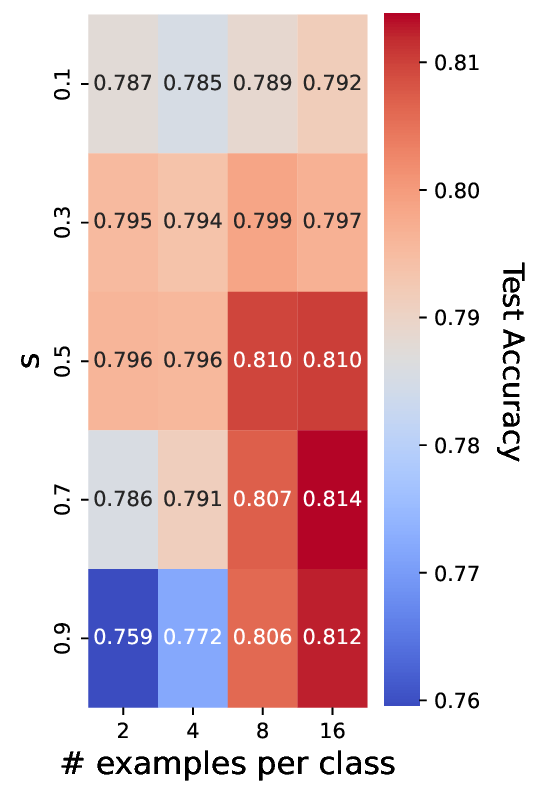}
}
    \caption{Target data Size and the selection of the mixing weight $s$. Recall that $s$ represents the weight assigned to the target data. Consistent with our intuition and the analysis presented in Sections \ref{sec: theory_subpop} and \ref{fig: subpopshift_curve}, we observe that with a smaller target data size, the optimal value of $s$ decreases. }
    \label{fig: s_size}
\end{figure}

\subsection{Comparisons with \cite{zhu2022progressive} and \cite{zhang2022few}}\label{apdx: additional_comp}

\cite{zhu2022progressive} employs both image-level and representation-level mixup, with the mixing ratio progressively adjusted, and optimization performed using the MAML meta learning framework. Below, we present a comparison between their method and ours.
\begin{itemize}
    \item \textbf{\ours\ is data-efficient and performs better.} Our method is  more data-efficient than \cite{zhu2022progressive}, achieving better performance when given very few target data. We conducted new experiments on PACS (with the target domain being S) and Office Home (with the target domain being R) using ResNet18, employing the same setup as in \cite{zhu2022progressive} (which uses only a subset of the classes), and reported our method’s mean Average Precision (mAP) — the metric used in their paper — comparing it with theirs in Table \ref{tab: zhu2022}. We observed that our method indeed outperforms theirs on both datasets. This could be attributed to their approach of updating the encoder rather than freezing it, which can potentially be harmful as the entire network is more prone to overfitting the very few target examples, compared to just the linear layer.
    \begin{table}[]
        \centering
        \begin{tabular}{|c|c|c|}
        \hline
            dataset & P-Mixup & \ours \\
            \hline
           PACS  & 81.97 & {\bf 84.50} \\
           Office-Home & 85.05 & {\bf 90.84}\\
           \hline
        \end{tabular}
        \caption{Comparison between \ours\ and \cite{zhu2022progressive} on Office-Home and PACS.}
        \label{tab: zhu2022}
    \end{table}
    \item \textbf{\ours\ is lightweight.} A few factors contribute to their method being significantly more costly than ours: (1) They update the encoder and employ image-level mixup, while the cost of our method is equivalent to just linear probing. (2) Their use of MAML, a meta learning framework, involves computing the gradient of a gradient, which is very expensive even with approximations. Despite being considerably more cost-effective, our method still achieves superior performance.
\end{itemize}

LCCS \cite{zhang2022few} adapts the normalization layers leveraging both source data and the information from the few target data.  Below, we present a comparison between their method and ours.
\begin{itemize}
    \item \textbf{\ours\ is data-efficient and performs better.} Our method is more data-efficient than \cite{zhang2022few} and achieves better performance given few target data. We conducted new experiments on the PACS dataset using ResNet18, following the same setup as in \cite{zhang2022few}. The performances are compared in Table \ref{tab: zhang2022few}. \ours\ outperforms theirs with the margin being especially significant when the amount of target data is small. The potential reason is that batch normalization, known to speed up training, also increases the risk of overfitting given few target data, thereby potentially degrading representation quality. In contrast, our method only trains the last linear layer, posing smaller risk of overfitting.
    \begin{table}[]
        \centering
        \begin{tabular}{|c|c|c|c|c|c|c|}
        \hline
           &  LCCS (1-shot)	 & MixPro (1-shot)	& LCCS (5-shot)	& MixPro (5-shot) &	LCCS (10-shot)	& MixPro (10-shot) \\
           \hline
           Sketch  &  72.5 &	{\bf 75.4} &	76.7	& {\bf 77.1} & 	79.4 &	{\bf 79.5}\\
           \hline
        \end{tabular}
        \caption{Comparison between \ours\ and \cite{zhang2022few} on PACS.}
        \label{tab: zhang2022few}
    \end{table}
    \item \textbf{\ours\ is lightweight.} Our method is also more efficient compared to \cite{zhang2022few}. Their method involves performing SVD on the matrix of channel-wise features, which is extremely expensive for very large models, whereas our method solely trains a linear layer on representations.
\end{itemize}

\end{document}